\documentclass[twoside]{article}
\usepackage[accepted]{aistats2016}
\usepackage[numbers,sort&compress]{natbib}

\usepackage{url}
\usepackage{graphicx}
\usepackage{amsmath}
\usepackage{amssymb}

\usepackage{bm} 
\usepackage{comment}

\runningauthor{Hauberg, Freifeld, Boesen Lindbo Larse, Fisher III, and Hansen}

\renewcommand{\vec}[1]{\mathbf{#1}}

\DeclareMathOperator{\Exp}{Exp}
\DeclareMathOperator{\Log}{Log}
\DeclareMathOperator{\dist}{dist}

\DeclareMathOperator*{\argmin}{arg\,min}
\newcommand{\T}{^{\top}}

\newtheorem{lem}{Lemma}
\newenvironment{proof}[1][Proof]{\begin{trivlist}
\item[\hskip \labelsep {\bfseries #1}]}{\end{trivlist}}

\newcommand{\qed}{\nobreak \ifvmode \relax \else
      \ifdim\lastskip<1.5em \hskip-\lastskip
      \hskip1.5em plus0em minus0.5em \fi \nobreak
      \vrule height0.75em width0.5em depth0.25em\fi}

\newcommand{\RR}{\ensuremath{\mathbb{R}}}
\newcommand{\Rtwo}{\ensuremath{\RR^2}}
\newcommand{\Rd}{\ensuremath{\RR^d}}

\newcommand{\set}[1]{\ensuremath{{\{#1\}}}}

\bmdefine\bSigma{\Sigma}

\bmdefine\bzero{0}
\bmdefine\bbeta{\beta}
\bmdefine\bphi{\phi}
\bmdefine\btheta{\theta}
\bmdefine\bxi{\xi}
\bmdefine\bn{n}
\bmdefine\bv{v}
\bmdefine\bx{x}

\bmdefine\bA{A}
\bmdefine\bB{B}

\newcommand{\Ncal}{\mathcal{N}}

\newcommand{\Pcal}{\mathcal{P}}

\newcommand{\Vcal}{\mathcal{V}}

\newcommand{\phiTheta}{\bphi^\btheta}
\newcommand{\phiThetaAt}[2]{\phiTheta(#1,#2)}

\newcommand{\Vpa}{\Vcal_{\mathrm{PA}}}

\newcommand{\SigmaPA}{\bSigma_{\mathrm{PA}}}

\begin{document}

  \twocolumn[
\runningtitle{Class-dependent Distributions over Diffeomorphisms for Learned Data Augmentation}
\aistatstitle{Dreaming More Data: 
      Class-dependent Distributions \\ over Diffeomorphisms for Learned Data Augmentation}
    \aistatsauthor{S{\o}ren Hauberg \And Oren Freifeld \And Anders Boesen Lindbo Larsen}
    \aistatsaddress{Section for Cognitive Systems\\Technical University of Denmark\\\texttt{sohau@dtu.dk} \And
      Sensing, Learning and\\ Inference Group\\MIT CSAIL\\\texttt{freifeld@csail.mit.edu} \And
      Image Analysis and Computer \\ Graphics Section\\Technical University of Denmark\\\texttt{abll@dtu.dk}}
    \aistatsauthor{John W.\ Fisher III \And Lars Kai Hansen}
    \aistatsaddress{Sensing, Learning and Inference Group\\MIT CSAIL\\\texttt{fisher@csail.mit.edu} \And
      Section for Cognitive Systems\\Technical University of Denmark\\\texttt{lkai@dtu.dk}}
  ]

  \begin{abstract}
  \emph{Data augmentation} is a key element in training high-dimensional models.
  In this approach, one synthesizes new observations by applying pre-specified transformations
  to the original training data; e.g.~new images are formed by rotating old ones.
  Current augmentation schemes, however, rely on manual specification of the applied transformations, 
  making data augmentation
  an implicit form of feature engineering. With an eye towards true \emph{end-to-end
  learning}, we suggest \emph{learning the applied transformations on a per-class basis}.
  Particularly, we align image pairs within each class under the assumption that the
  spatial transformation between images belongs to a large class of diffeomorphisms.
  We then learn a class-specific probabilistic generative models of the transformations in a Riemannian
  submanifold of the Lie group of diffeomorphisms. We demonstrate significant
  performance improvements in training deep neural nets over manually-specified
  augmentation schemes. 
  Our code and augmented datasets are available online.
\end{abstract}

  \section{Introduction}\label{sec:intro}
  Variation in classification datasets typically reflects a variety of physical processes.
  Some are related to label differences and others inject variability which is
  irrelevant to the labeling. In visual object recognition we face different objects,
  but also different locations, postures, variable illumination conditions, and so forth.
  Mathematically this means that the image label is invariant to
  certain transformations of the image.
  Knowledge of these transformations can be used to drastically improve classification performance,
  either by picking invariant features or by augmenting the dataset with new
  observations generated by applying random transformations to the training data.
  This form of \emph{data augmentation} is a rather old idea~\cite{baird1992document,simard1992tangent},
  which is still considered part of ``best practices''~\cite{simard2003best} for many tasks.
  Data augmentation was a critical component in the landmark image-classification
  paper by Krizhevsky et al.~\cite{HintonImageNet2012}:
  \emph{``Without this scheme, our network suffers from substantial overfitting,
  which would have forced us to use much smaller networks''}.
  
  In practice, data augmentation is a manual process, where a human specifies
  a small set of transformations for which an image classification task is believed to be invariant; 
  for image classification tasks, these are most commonly chosen to be simple linear transformations such as 
  \emph{translations}, \emph{rotations} and \emph{scaling}. In essence, this manual process is
  a form of \emph{feature engineering}, where the data itself 
  plays only a secondary role. 
  Moreover, as current augmentation schemes are manually specified, the same scheme
  is often used within all classes. This is potentially troublesome, e.g.\ while
  full \emph{rotational invariance} is poorly suited when differentiating images
  of digits ``6'' and ``9'', it may still help differentiate other digits.
  These concerns suggest
  that it is beneficial to \emph{learn augmentation schemes}
  from training data rather than relying on manual specification, thus, getting
  closer to an \emph{end-to-end learning scheme}.
  
  We consider an unsupervised approach for learning image deformations
  that naturally appear within different classes. 
  For each class, we align images in a pairwise fashion, assuming the latent
  spatial transformation between them is a ($C^1$) diffeomorphism.
  This gives a set of diffeomorphisms represented
  on a finite-dimensional nonlinear Riemannian manifold. We show that our data has a simple
  closed-form mean on this manifold, and approximate the distribution of
  diffeomorphisms with a per-class multivariate normal distribution in
  the tangent space at the mean. We then generate new data
  by sampling first an image from the training data, and second a diffeomorphism
  from the learned distribution.
  Applying the diffeomorphism to the image gives a new data point.
  With this data we train both a multilayer perceptron and a convolutional neural net.
  In both cases we observe a significant
  improvement over manually specified augmentation schemes. This is particularly
  evident for small datasets.
  Both the code and generated data are available 
online.\footnote{\url{http://www2.compute.dtu.dk/~sohau/augmentations/}}
    
  \begin{figure*}[t]
    \includegraphics[width=\textwidth]{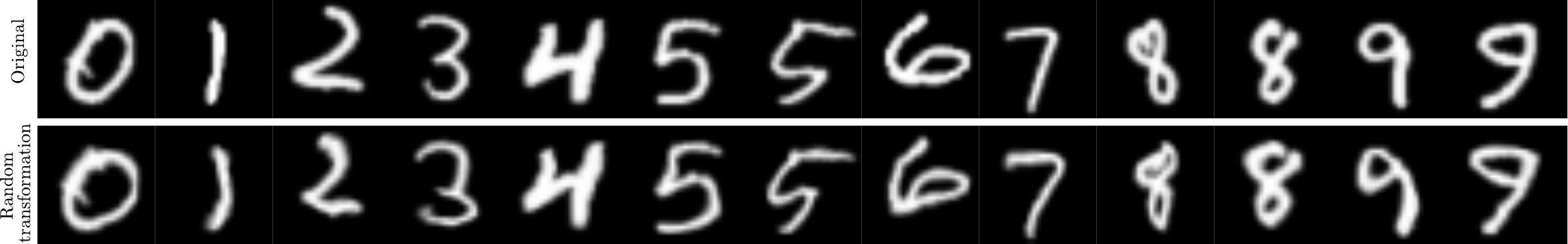}
    \caption{Example transformations.
      \emph{Top:} original images.
      \emph{Bottom:} random transformations from our model applied to the original images.}
    \label{fig:examples}
  \end{figure*}

  \section{Background and Related Work}\label{sec:background}
  We focus the discussion on image classification, though the fundamental ideas
  generalize to other domains.
    Let $T^{\btheta}$ 
  be a spatial transformation, parametrized by $\btheta$,
  such that if an image $\bx$ has label $y$, then
  so does  $\bx \circ T^{\btheta} $.
  For example, translating an
  object in an image does not change its label. If $T^{\btheta}$ is known
  then this conveys strong prior information, which may improve a classifier.
  This is a key consideration when training deep models, cf.~\cite{goodfellow2009measuring}.

  One strategy for encoding knowledge of $T^{\btheta}$ is to design features that are
  invariant to $T^{\btheta}$, 
  i.e.~$\mathcal{F}({\bx}) = \mathcal{F}({\bx} \circ T^{\btheta})$,
  where $\mathcal{F}$ denotes the feature map. Equivalently, invariant kernels~\cite{chapelle:nips:2001} 
  can be derived. This idea lies at the heart of
  Grenander's \emph{Pattern Theory}~\cite{grenander:book:1993}, which encodes invariances
  by describing the data itself as transformations acting on reference objects.
  Alternatively, all instantiations of $T^{\btheta}$ may be applied to each observation
  to produce orbits~\cite{graepel:nips:2004, liao:nips:2013} which can then be matched.
  While mathematically elegant, these approaches tend to be laborious and computationally
  expensive, which hinder their applicability. Moreover, the resulting transformations
  generally have limited expressiveness (partly to keep computations tractable,
  the transformation family is usually restricted to be simple; e.g., affine transformations).
  
  To counter this, many approaches are only approximately invariant to 
  the transformations. For example, scale invariance is approximated via image-pyramid
  representations~\cite{kanazawa:arxiv:2014, farabet:pami:2013}, assuming
  a known set of scales which the classifier must be invariant to.
  The classic \emph{Tangent Prop}~\cite{simard1992tangent}
  locally linearizes $T^{\btheta}$ with finite differences, and forces the back-propagated
  gradient of a neural net to respect the invariance. The linearization, however,
  implies that invariance can 
  be learned  
  with respect to only infinitesimal 
  transformations. General linear invariances 
  are also used 
  for restricted
  Boltzmann machines~\cite{sohn:icml:2012, kivinen:icann:2011}, but again the
  linearity implies that the invariance is only infinitesimal.
  
  A widely-used approach is to synthesize new observations by applying
  the known $T^{\btheta}$ to the training data, and then train a classifier on
  the augmented data set~\cite{baird1992document,simard2003best,HintonImageNet2012,loosli:lskm:2007}.
  Eigen et al.~\cite{NIPS2014_5539} consider multiple invariances including length scaling,
  rotation, translation, horizontal flip and color scaling in a deep multiscale network
  for single-image depth estimation. In work on speech recognition, Jaitly and Hinton~\cite{jaitly2013vocal} apply
  \emph{Vocal Tract Normalization (VTLN)}
  as a way to artificially transform utterances of one speaker to the voice of another.
  The transform is applied in the spectral domain and corresponds
  in the VTLN model to a simple parameterized warp of the frequency axis~\cite{jaitly2013vocal}.
  In astrophysics, in the context of galaxy redshift prediction, Hoyle et al.~\cite{Hoyle:2015ada}
  augmented data by use of redshift models.
  The \emph{Infinite MNIST} data set~\cite{loosli:lskm:2007}
  was 
  generated by considering horizontal and vertical translations, rotations,
  horizontal and vertical scalings, hyperbolic transformations, and random Gaussian
  perturbations. Using the ideas
  of Tangent Prop, infinitesimal transformations are then applied to the training
  data to produce a total of 8 million observations. This is the current most elaborate
  augmentation strategy.
  
  Semi-supervised learning~\cite{chapelle2006semi} in data with a neighborhood
  graph can be related to model-based data augmentation,
  the modeling assumption
  being that neighbors share the same label; for examples in text classification, see~\cite{lu2006enhancing}. 
  In an interesting twist on data augmentation, Dosovitskiy et al.~\cite{dosovitskiy2013unsupervised}
  create various augmented data, and learn features by training networks to
  \emph{distinguish between} the differently augmented datasets.

  The above approaches all assume that the transformations $T^{\btheta}$ are known \emph{a priori},
  which is generally not the case. The specification of \emph{which} transformation
  should be considered is an implicit form of feature engineering, and it is worth
  investigating whether the transformations themselves should be learned from data.
  \emph{This is the objective of the current manuscript}.
  
  Transformations between observations can also be used to define a similarity
  measure between observations by quantifying the ``size'' of the transformation.
  This was used with a $k$-nearest neighbors classifier by Amit et al.~\cite{Amit:ijcv:2007}
  to create the current state-of-the-art MNIST classifier relying on only a low
  number of training images. Similarly, Duchenne et al.~\cite{duchenne2011graph}
  use the size of transformations to build a kernel for SVM classification.
  To speed up the estimation of the transformations, both Winn \& Jojic~\cite{winn2005locus}
  and Hariharan et al.~\cite{hariharan2014detecting} learn models of transformations to guide the search.

  \section{Diffeomorphisms: Representation, Inference, and Learning}
To learn a model $p(T^{\btheta} | y)$ of the transformations within each class,
we need a well-behaved and sufficiently-expressive mathematical representation of transformations.
To avoid a manual specification of the transformations of interest (such as rotations, translations, etc.),
we focus our attention on a large class of first-order 
diffeomorphisms, i.e.\ we consider $T^{\btheta}$ that are once differentiable, 
and whose inverse, $(T^{\btheta})^{-1}$, exists and is differentiable as well. 

    \begin{figure*}
      \centering
      \includegraphics[width=\textwidth]{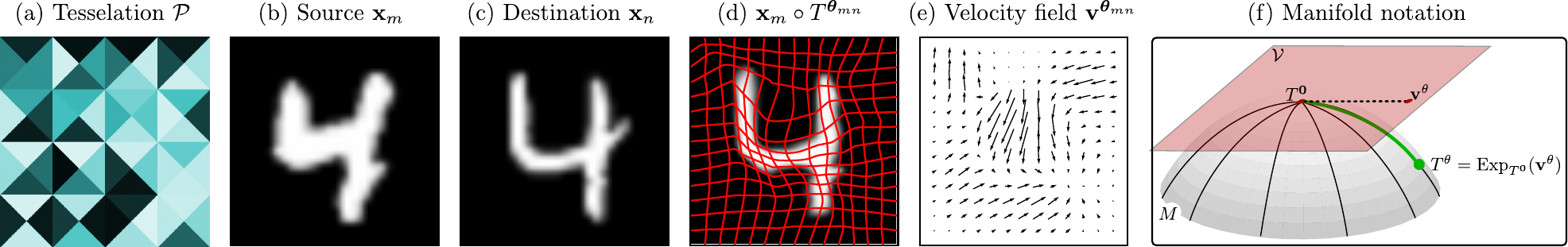}
      \caption{(a) An example tesselation.
        (b--c) Two images from the MNIST dataset.
        (d) The transformation between the two images.
        (e) The velocity field of the transformation.
        (f) Illustration of Riemannian concepts: $\Vcal$ is the tangent space
            of the manifold $M$ at the identity transformation;
            the exponential map take points from $\Vcal$ onto $M$.}
      \label{fig:alignment}
    \end{figure*}
    
  \subsection{Representing Diffeomorphisms}
    The space of \emph{all} diffeomorphisms is an infinite-dimensional Lie group \cite{arnold1989mathematical}.
    To reduce the implied computational burden when aligning images, we restrict
    our attention to a large finite-dimensional subset of this Lie group.
    Here we rely on recent developments in the image transformation literature~\cite{freifeld2015transform}:
    
    These developments specify a transformation through the integration of a \emph{velocity field
    with a simple structure} that, while being highly-expressive, leads to fast
    and accurate computations, thus rendering inference tractable.
    Let $\Omega\subset\Rtwo$ denote the image domain and let $\Pcal$ be a
    triangular tessellation of $\Omega$ (see Fig.~\ref{fig:alignment}a for an example).
    In what follows, an $\Omega \to \Rtwo$ velocity field is called
    \emph{piecewise-affine (PA)} if it is affine when restricted to each triangle
    of $\Pcal$. Let $\Vcal$ denote the space of all \emph{Continuous Piecewise-Affine (CPA)}
    $\Omega \to \Rtwo$ velocity fields that vanish on $\partial\Omega$, the boundary
    of $\Omega$:
    \begin{align}\label{eqn:Vcpa}
    \begin{split}
      \Vcal \triangleq \{&\bv^\btheta: \bv^\btheta \text{ is an } \Omega\to\Rtwo \text{ CPA map (w.r.t.~}\Pcal\text{)}  \\
               &\text{and} \, \bv^\btheta(\bxi)=\bzero_{2\times1}\,  \forall \bxi\in\partial\Omega
              \} \, .
    \end{split}
    \end{align}
    It can be shown~\cite{freifeld2015transform} that $\Vcal$ is a finite-dimensional linear space, where
    $d \triangleq \dim(\Vcal)$ is determined by how finely $\Pcal$ is tessellated; thus, $\Vcal$ is isomorphic to $\Rd$.
    In our experiments we use the tessellation shown
    in Fig.~\ref{fig:alignment}a. 
    This choice, together with the boundary condition, implies that $d = 50$
    (25 non-boundary vertices times 2 degrees of freedom in each one).
    In Eq.~\ref{eqn:Vcpa}, the superscripted $\btheta$ is an element 
    of $\Rd$, and we say that $\btheta$ parametrizes $\Vcal$ in the sense 
    that any element of $\Vcal$ can be written as a linear combination of $d$
    orthonormal CPA fields with weights $\btheta$~\cite{freifeld2015transform}.

    A spatial transformation $T^\btheta:\Omega\to\Omega$ can then be derived by integrating
    a velocity field $\bv^\btheta \in \Vcal$ \cite{freifeld2015transform};
    in other words, the space of the resulting transformations, denoted by $M$, is given by
    \begin{align}
    \begin{split}
      M \triangleq \{ &T^\btheta:\Omega \to \Omega\,,\,T^\btheta:\bxi \mapsto \phiThetaAt{\bxi}{1}
        \text{ where } \\
        &\phiThetaAt{\bxi}{\cdot} \text{ solves Eq.~\ref{eqn:intEqn} for some }\bv^\btheta \in \Vcal \} \, ;
    \end{split}
    \end{align}
    \begin{align}
    \begin{split}
      \phiThetaAt{\bxi}{t} = \bxi + \int_{0}^{t} \bv^\btheta ( \phiThetaAt{\bxi}{\tau} ) \, d  \tau \, ,
      \\
      \phiThetaAt{\cdot}{\cdot}:\Omega\times\RR \to \Omega \, .
      \label{eqn:intEqn}
    \end{split}
    \end{align}
    It can be shown \cite{freifeld2015transform} that elements of $M$ are 
    diffeomorphisms, that $M$ contains the identity
    transformation $T^{\bzero}:\bxi\mapsto\bxi$, and that $M$ is closed under inversion with
    $(T^{\btheta})^{-1} = T^{-\btheta}$.
    While these elements are obtained, via the solution of Eq.~\ref{eqn:intEqn},
    from CPA velocity fields, they are generally \emph{not} CPA themselves.
    Thus, elements of $M$ are called CPA-\emph{based} (CPAB) transformations \cite{freifeld2015transform}.

    It can further be shown that $M$ is a nonlinear
    $d$-dimensional connected Riemannian manifold \cite{freifeld2015transform},
    and that $\Vcal$ is a tangent space
    of $M$, where the identity transformation $T^{\bzero}$ is the point of
    tangency.    
    To build statistical models over $M$ we need mappings back and forth
    between the manifold and its tangent spaces \cite{pennec:jmiv:2006}.
    It can be shown that the restriction of the Lie group exponential map 
    to $\Vcal$ coincides with the mapping of $\bv^\btheta$ to $T^\btheta$ via
    the integral equation \eqref{eqn:intEqn}. 
    We 
    thus
    write $T^\btheta = \Exp_{T^{\bzero}} (\bv^\btheta)$,
    where 
    $\Exp_{T^{\bzero}}:\Vcal\to M$
    denotes the exponential mapping.
    Figure~\ref{fig:alignment}f illustrates the notation.
    It is worth noting that due to the special structure of CPA velocity fields, the exponential map has an efficient numerical
    implementation, which makes inference tractable \cite{freifeld2015transform}.
    Particularly, this enables a fast GPU implementation 
    that parallelizes the evaluation of $(\bxi,\btheta)\mapsto T^\btheta(\bxi)$  over multiple $\bxi$'s.
    We end this section with two remarks.
    First, as noted in~\cite{freifeld2015transform}, $\bxi\mapsto \phiThetaAt{\bxi}{t=1}$ should
    not be confused with (the non-diffeomorphism, parametric optical-flow-like representation) $\bxi\mapsto\bxi +\bv^\btheta(\bxi)$,
    the latter being only a \emph{Taylor approximation} of the former.
    Second, while here we used the transformations from~\cite{freifeld2015transform} due the reasons mentioned above,
    one may also consider other finite-dimensional spaces of diffeomorphisms~\cite{zhang:ipmi:2015,Arsigny:BIR:2006}. 
    
  \subsection{Estimating Transformations by Aligning Images}
    To estimate the transformations that naturally appear within a class, we
    align image pairs from said class. To avoid aligning all image pairs, we consider only the image pairs 
    where one image is among the $K=5$ nearest neighbors of the other.    
    Let $\bx_n$ and $\bx_m$ denote two such images. We then seek
    $T^{\btheta_{mn}} \in M$ such that  
     $\bx_m \circ T^{\btheta_{mn}} \approx \bx_n$.
    
        We adopt a Bayesian approach where our prior encodes
    smoothness. The posterior is given by  
    \begin{align}
      p(T^{\btheta_{}} | \bx_m, \bx_n)
        &\propto p( \bx_m, \bx_n | T^{\btheta_{}})\, p(T^{\btheta_{}}).
    \end{align}
    For the likelihood we take a simple i.i.d.\ Gaussian model of the intensity
    differences between the destination image and warped one:
    \begin{align}
      \log p( \bx_m, \bx_n | T^{\btheta_{}})
       &= -\frac{\| \bx_m\circ T^{\btheta_{}} - \bx_n\|^2}{2\sigma^2} + \mathrm{const},
    \end{align}
    where $\sigma$ is user-specified. Other similarity measures,
    such as mutual information~\cite{viola1997alignment}, can be used as well.

    \begin{figure*}
      \centering
      \includegraphics[width=1.0\textwidth]{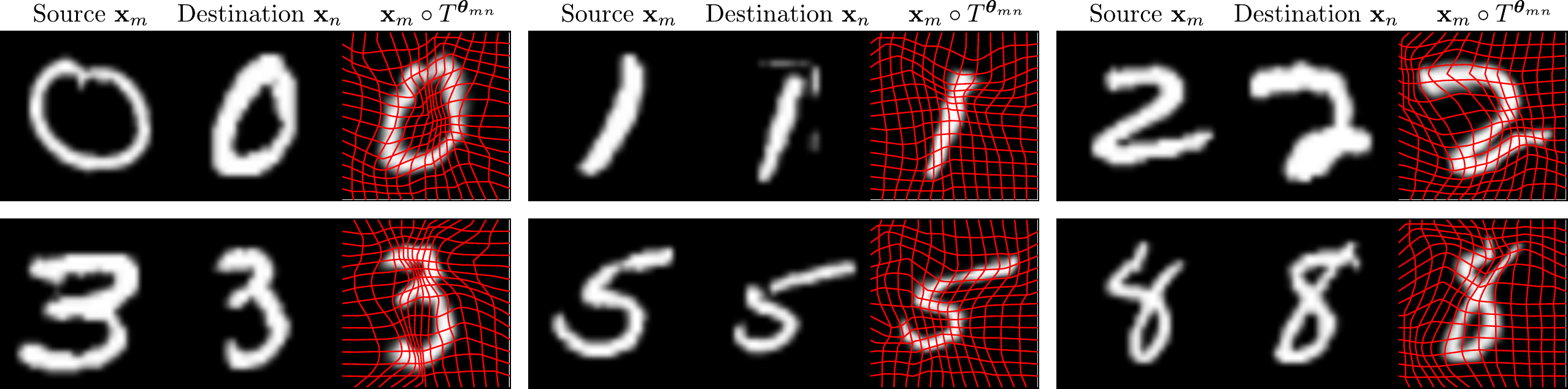}
      \caption{Example alignments.}
      \label{fig:align_examples}
    \end{figure*}
  
    Following~\cite{freifeld2015transform}, we construct the prior over velocity fields $\bv^\btheta \in \Vcal$, which
    induces a prior over transformations via the exponential map,
    $T^\btheta = \Exp_{T^{\bzero}}(\bv^\btheta)$.
    Let $\Vpa$ denote the linear space of (possibly-discontinuous, and without
    boundary constraints) PA $\Omega \to \Rtwo$ velocity fields. 
    The affine map associated with each triangle is given by a $2 \times 3$ matrix
    so $\dim(\Vpa) \triangleq D = 6 N_c$, where $N_c$ is the number of
    triangles in $\Pcal$. We use a standard \emph{squared-exponential}
    (aka.\ \emph{Gaussian}) kernel to define a zero-mean Gaussian distribution
    over $\Vpa$ such that the correlations between the affine maps associated
    with each pair of triangles decay with the distance between triangle centroids \cite{freifeld2015transform}.
    Let $\SigmaPA$ denote the covariance of that Gaussian. Note that $\Vcal$
    is a linear subspace of $\Vpa$ with $d = \dim(\Vcal) < D$.
    Let $\bB$ be the $D \times d$ orthogonal basis matrix associated with the weights
    $\btheta$. The prior over $\bv^\btheta$ is then defined as the projection of the
    prior over $\Vpa$,
    \begin{align}
      &\bv^\btheta \cong \btheta \sim \Ncal(\bzero, \bB\T \SigmaPA \bB)
      \\
       &\text{(the symbol $\cong$ indicates that $\bv^\btheta$ is isomorphic to $\btheta$)}\, . 
       \nonumber
    \end{align}
    We remark that other priors are easily implemented, e.g.\ a
    Gaussian Markov Random Field (GMRF) could provide an (improper) prior that
    prefers smoothness without penalizing large offsets.

    For inference, we adopt a sampling approach using a Markov-Chain Monte-Carlo (MCMC) method~\cite{Robert:Book:MCMC:2004}.    
    Particularly, our implementation uses the time-honored Metropolis algorithm~\cite{metropolis1953equation}
    with a (localized) Gaussian proposal.
    This is made tractable by the efficient numerical implementation of the exponential map~\cite{freifeld2015transform}
    and the moderate dimensionality of $\Vcal$.
    Using this sampling scheme, we then set $T^{\btheta_{mn}} = T^{\btheta} \sim p(T^{\btheta_{}} | \bx_m, \bx_n) $.
    
    Figure~\ref{fig:align_examples} shows examples of image alignments performed
    on the MNIST data set. It is evident that the estimated transformations
    are highly non-rigid, yet reasonably smooth. While the transformations generally
    provide a good fit, they are unable to capture topological differences. The
    rightmost column
    of Fig.~\ref{fig:align_examples} shows two such examples:
    one of the top 2-digits has a self-intersection, while the other does not;
    and one of the bottom 8-digits is open-ended at the top, while the other is
    not. In both cases, the transformation preserves the topology of the
    input image.
    
    The entire alignment process takes approximately 10 seconds when implemented
    on GPU hardware.


  \subsection{Statistical Models of Transformations}
Having aligned images pairs as discussed above, we now treat the inferred $\set{T^{\btheta_{mn}}}$ as observations. 
  We aim to build a statistical model of the transformations found within a
  class, i.e.\ $p(T^\btheta | y)$. Let $T^{\btheta_{nm}}$ denote the inferred CPAB transformation
  from image $\bx_n$ to $\bx_m$, and let $\btheta_{nm}\cong\bv^{\btheta_{nm}} = \Log_{T^{\bzero}}(T^{\btheta_{nm}})$
  be the representation of the transformation in the tangent space of the identity
  transformation; 
  we may identify
  this tangent space with $\RR^d$.
  Here $\Log_{T}(\cdot)$ is the inverse of $\Exp_{T}(\cdot)$.
  
  Common strategies for describing distributions over manifold-valued data include
  kernel density estimators \cite{miller:cvpr:2003}, and parametric models over tangent
  vectors \cite{pennec:jmiv:2006}. We have found the latter approach to work well
  for modeling the estimated transformations. To minimize the distortion due to
  the linearization
  of the manifold, it is common to use the tangent space at the intrinsic
  mean, defined as
  \begin{align}
    \mu = \argmin_{\tilde{\mu}} \sum_{nm} \dist^2(T^{\btheta_{nm}}, \tilde{\mu}),
    \label{eq:intrinsic_mean}
  \end{align}
  where $\dist^2(T^{\btheta_{nm}}, \tilde{\mu}) = \| \Log_{\tilde{\mu}}(T^{\btheta_{nm}}) \|^2$
  denotes the squared distance on the manifold. It is straight-forward
  to show that, in our case, the identity transformation is an intrinsic mean:
  \begin{lem}
    The collection of pair-wise transformations $\{ T^{\btheta_{nm}}, T^{\btheta_{mn}} \}_{n, m} \forall n, m$
    has an intrinsic mean \eqref{eq:intrinsic_mean}  $\mu = T^{\bzero}$.
  \end{lem}
  \begin{proof}
    The gradient of Eq.~\ref{eq:intrinsic_mean} is
    $\sum_{nm} \Log_{\tilde{\mu}}(T^{\btheta_{nm}})$~\cite{pennec:jmiv:2006}.
    Since $T^{\btheta_{nm}} = T^{-\btheta_{mn}}$,
    the gradient is zero at $T^{\bzero}$, the identity transformation, implying that this is an intrinsic mean.
    \hfill$\square$
  \end{proof}
  With this in mind, we build a tangential parametric model at the identity. 
  Empirically, we have
  found ${\bv^\btheta}|y \sim \mathcal{N}(\bzero, \bSigma_y)$
  to provide a good fit to the data. 
  Here $\bSigma_y = \frac{1}{N_y} \sum_{nm} (\bv^{\btheta_{nm}}) (\bv^{\btheta_{nm}})\T$
   denotes the Riemannian covariance of the transformations~\cite{pennec:jmiv:2006}, 
  where $N_y$ is the number of sample pairs used for this class.
  We then have
  \begin{align}
    p(T^\btheta | y)
      &\propto \exp\left(
        -\frac{1}{2} \Log_{T^{\bzero}}(T^\btheta)\T \bSigma_y^{-1} \Log_{T^{\bzero}}(T^\btheta)
      \right) \\
      &= \exp\left(
        -\frac{1}{2} ({\bv^\btheta})\T \bSigma_y^{-1} {\bv^\btheta}
      \right)\, .
  \end{align}

%
%
%
%

  \subsection{Sampling New Data}\label{sec:augmentation}
  We assume that transformations and images are conditionally independent given
  the class label $y$. 
  New transformations are generated as $T^{\btheta}_i = \Exp_{T^{\vec{0}}}(\bv^{\btheta}_i)$,
  where $\bv^{\btheta}_i \sim \mathcal{N}(\vec{0}, \Sigma_y)$.
  This transformation is then applied to a uniformly-chosen image from the
  training set, yielding a new image. 
  Note that by sampling not only the transformation but also the ``template'' image, the generated samples capture 
  the fact that, due to topological changes, some of the inter-class deformations are not diffeomorphic.

  Figure~\ref{fig:examples} shows randomly-chosen images from the MNIST dataset
  and randomly-deformed instances of the same images. It is evident that the
  randomly-sampled transformations are non-rigid, yet produce realistically-looking images. A more systematic 
illustration is given in Fig.~\ref{fig:PCs}.
  The shown images are $\pm 3$ times the standard deviation along the first principal
  component of the transformations found in each class. It is evident that $p(T^{\btheta} | y)$
  captures non-rigid deformations in the data that would be difficult to
  model manually, e.g.\ the first principal deformation of ``2'' captures
  whether the stroke of the tail bends upwards or downwards,
  while the first principal deformation of ``8'' captures the position of the self-intersection. 
  The supplementary material contains similar plots for higher-order components
  as well as an animation of these transformations. This animation provides the
  best visualization of the principal components.

  \begin{figure*}
    \centering
    \resizebox{0.9\textwidth}{!}{%
    \includegraphics[width=0.1\textwidth, clip=true, trim=4mm 4mm 4mm 4mm]{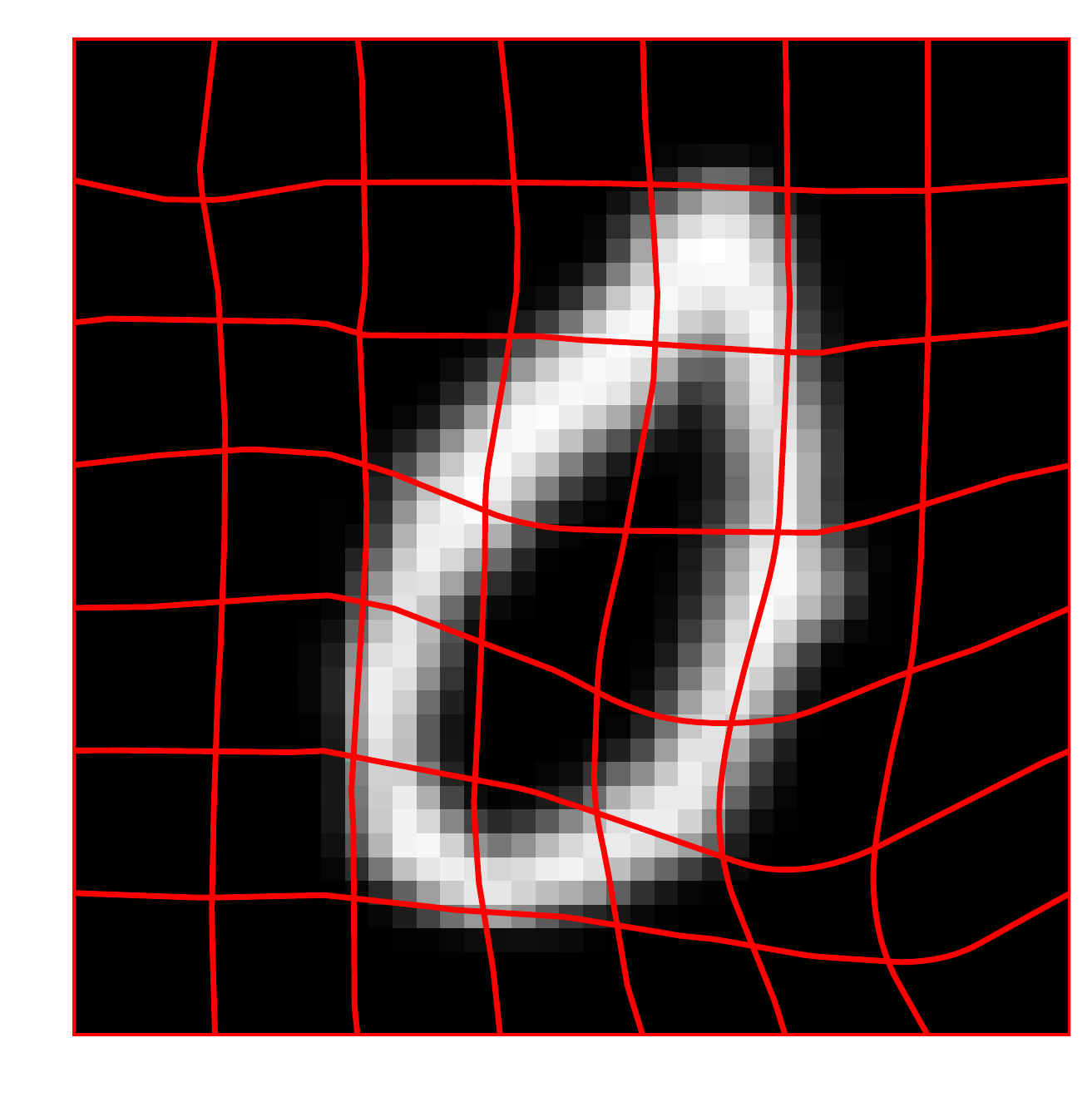}  \hspace{-3mm}
    \includegraphics[width=0.1\textwidth, clip=true, trim=4mm 4mm 4mm 4mm]{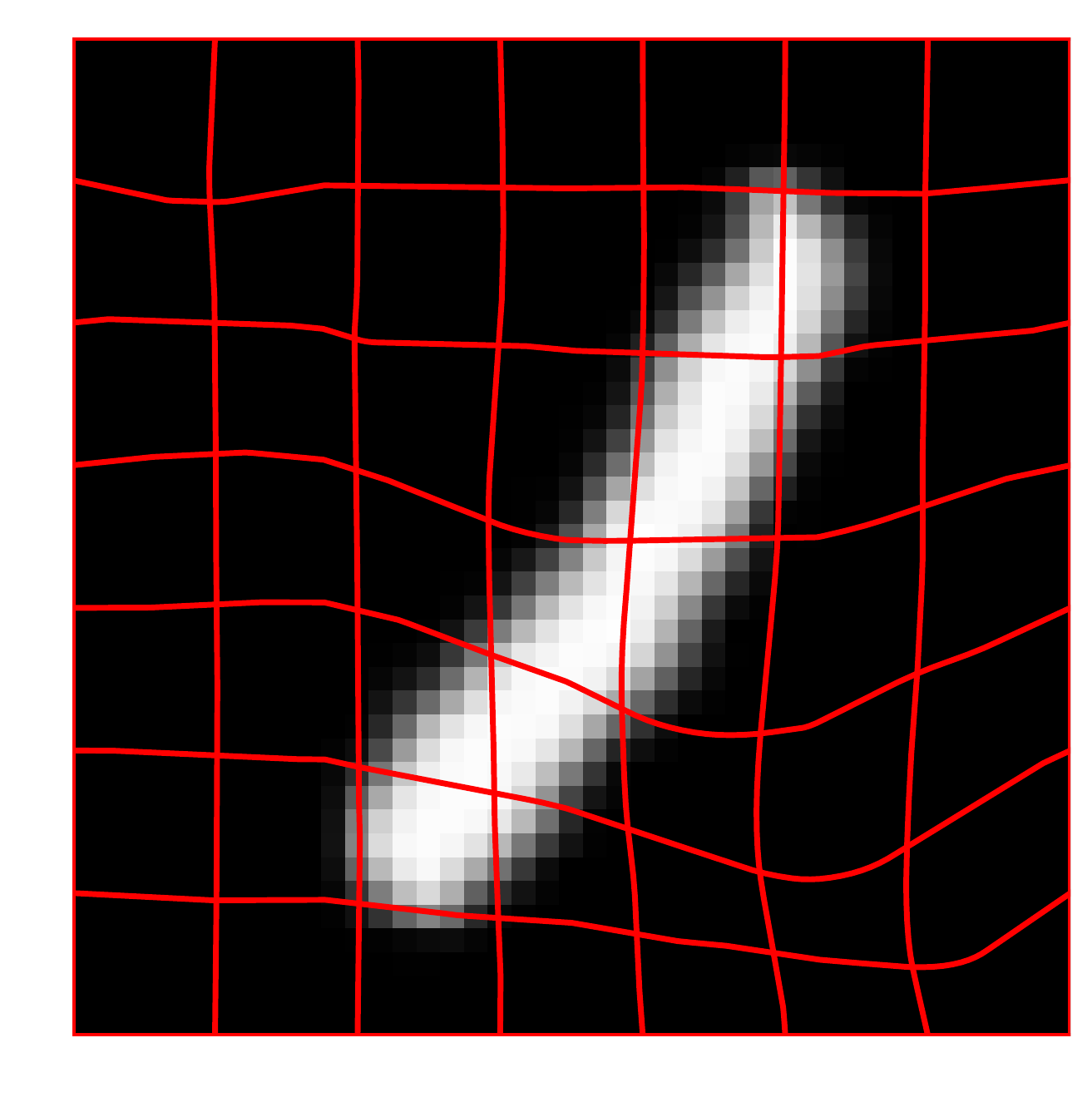}   \hspace{-3mm}
    \includegraphics[width=0.1\textwidth, clip=true, trim=4mm 4mm 4mm 4mm]{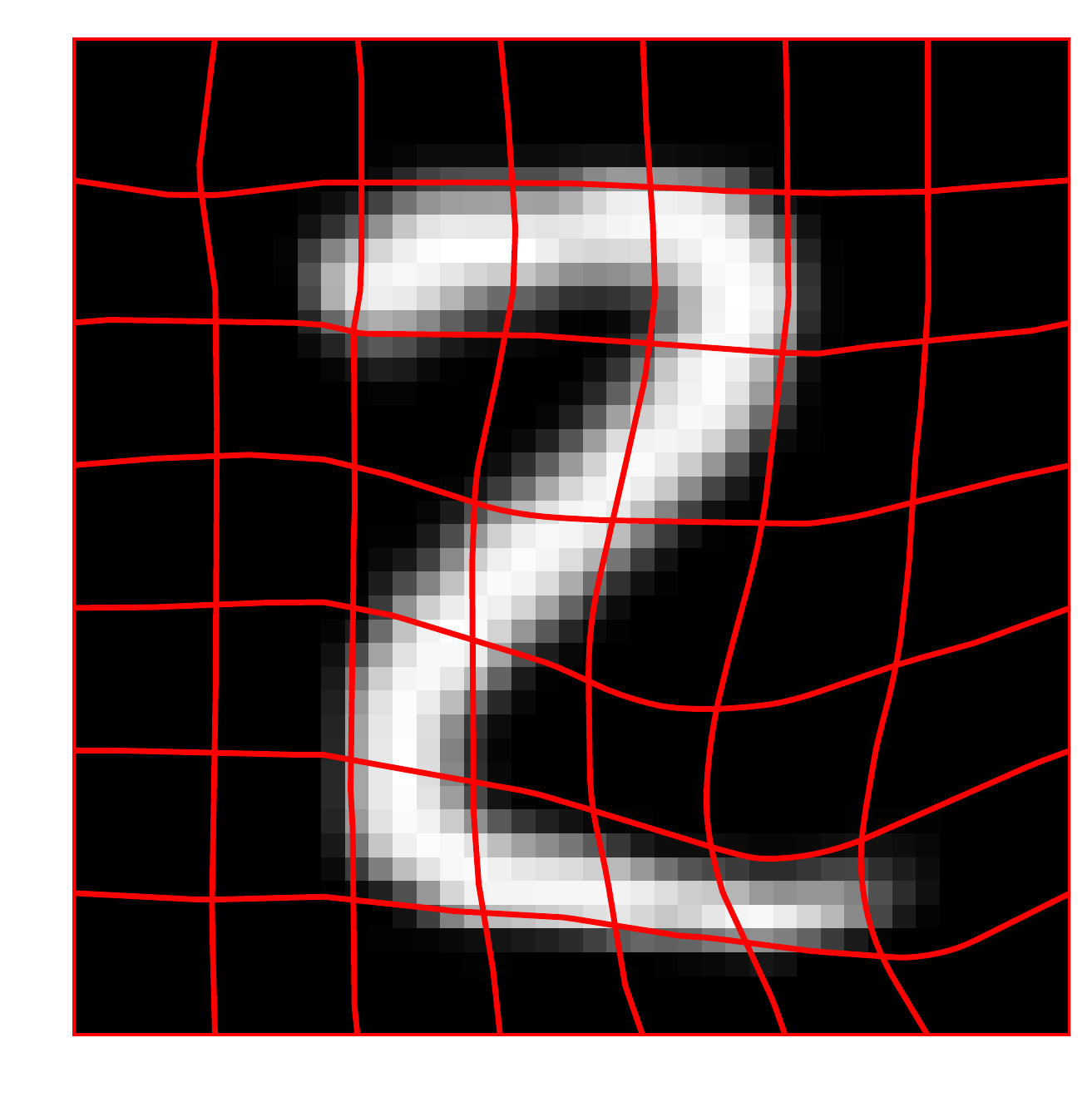}   \hspace{-3mm}
    \includegraphics[width=0.1\textwidth, clip=true, trim=4mm 4mm 4mm 4mm]{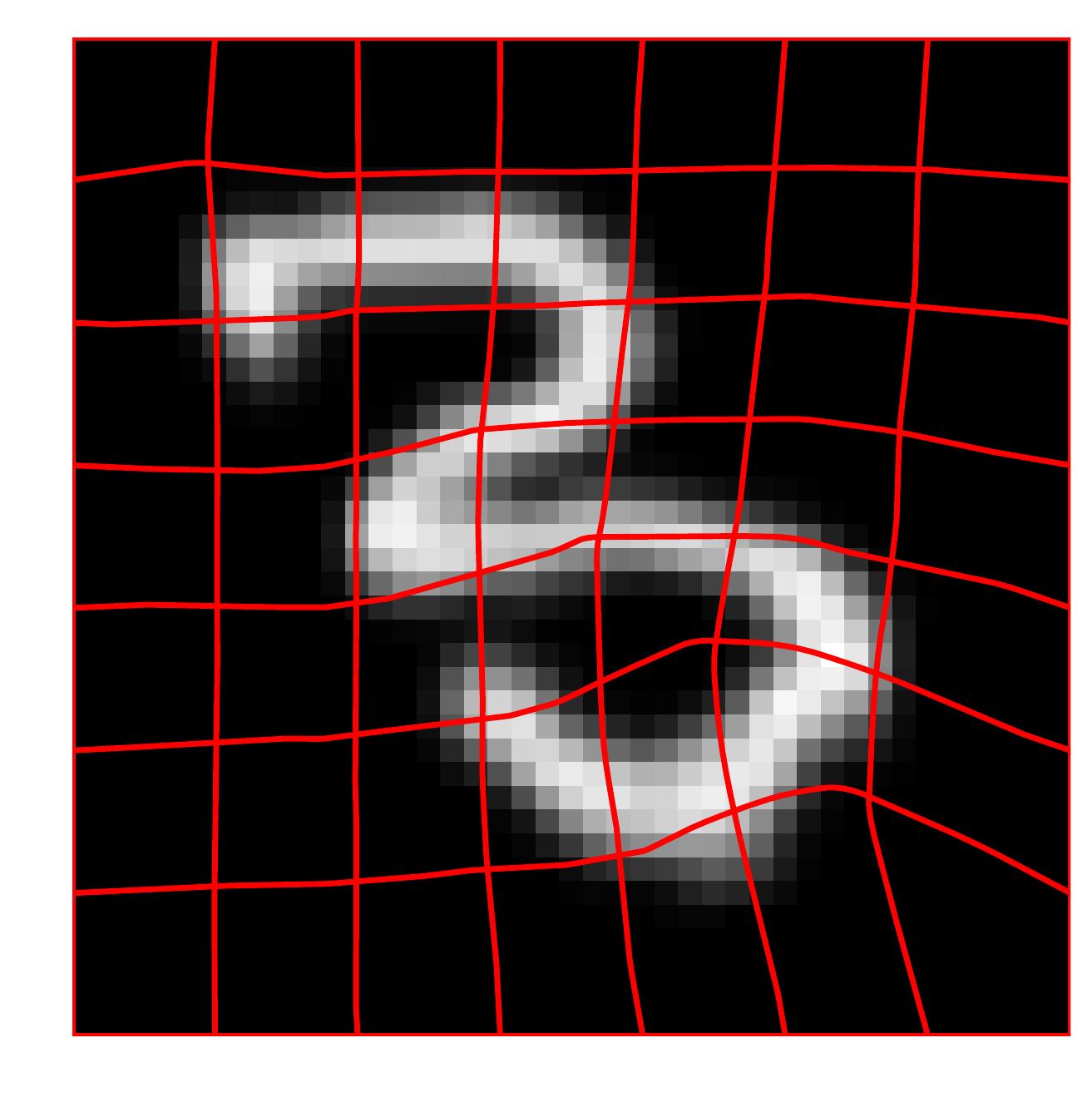} \hspace{-3mm}
    \includegraphics[width=0.1\textwidth, clip=true, trim=4mm 4mm 4mm 4mm]{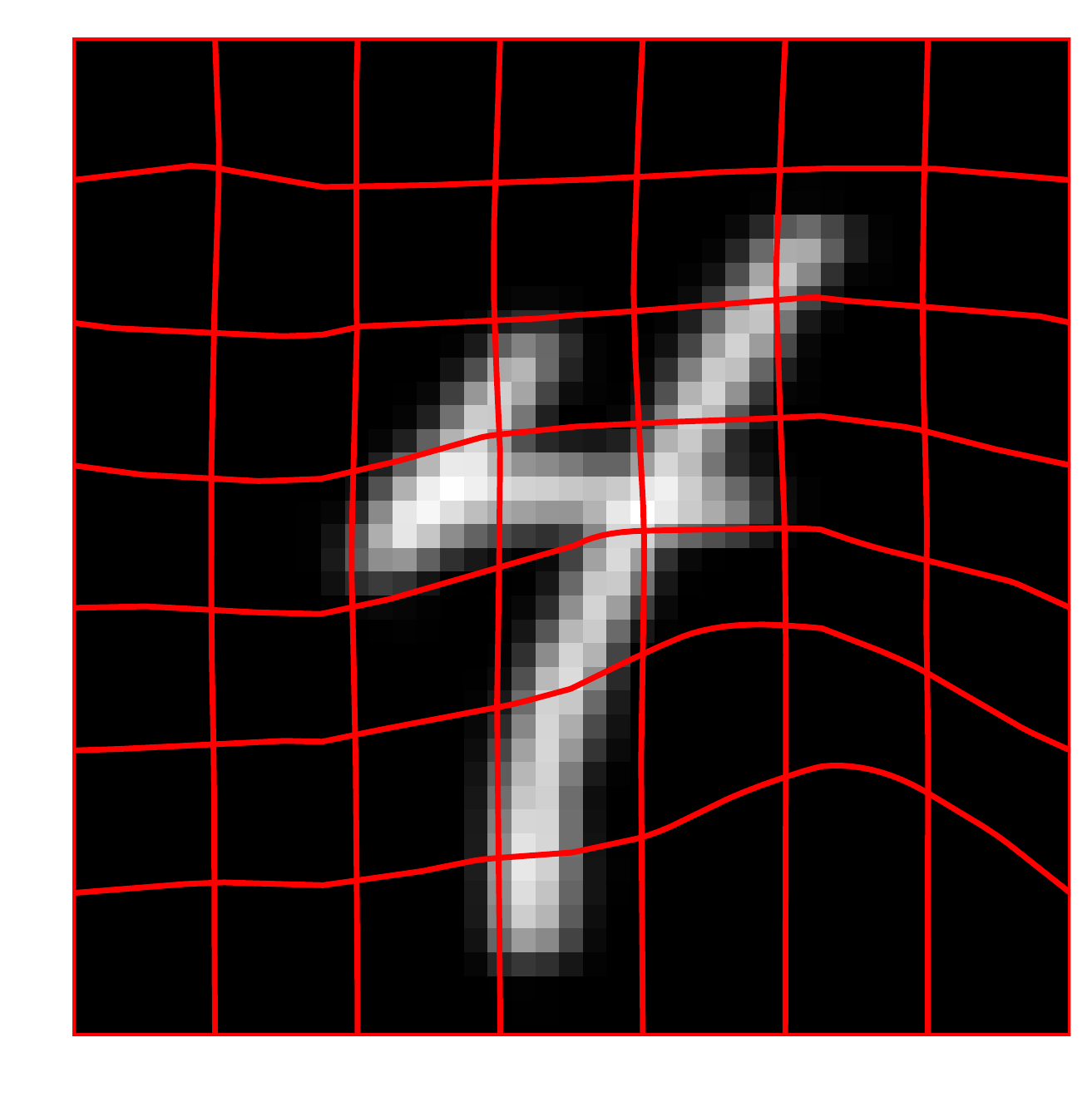}  \hspace{-3mm}
    \includegraphics[width=0.1\textwidth, clip=true, trim=4mm 4mm 4mm 4mm]{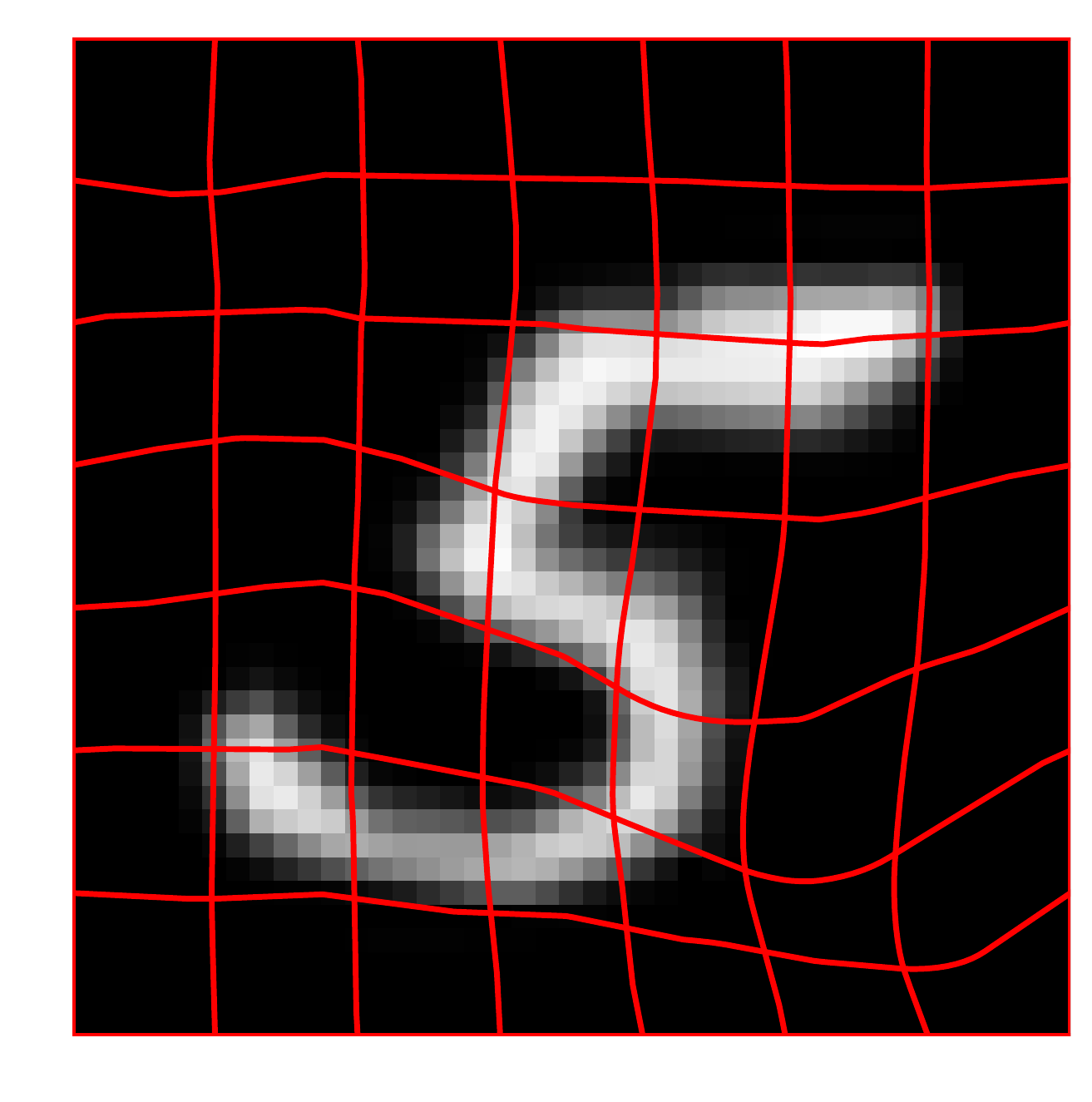}  \hspace{-3mm}
    \includegraphics[width=0.1\textwidth, clip=true, trim=4mm 4mm 4mm 4mm]{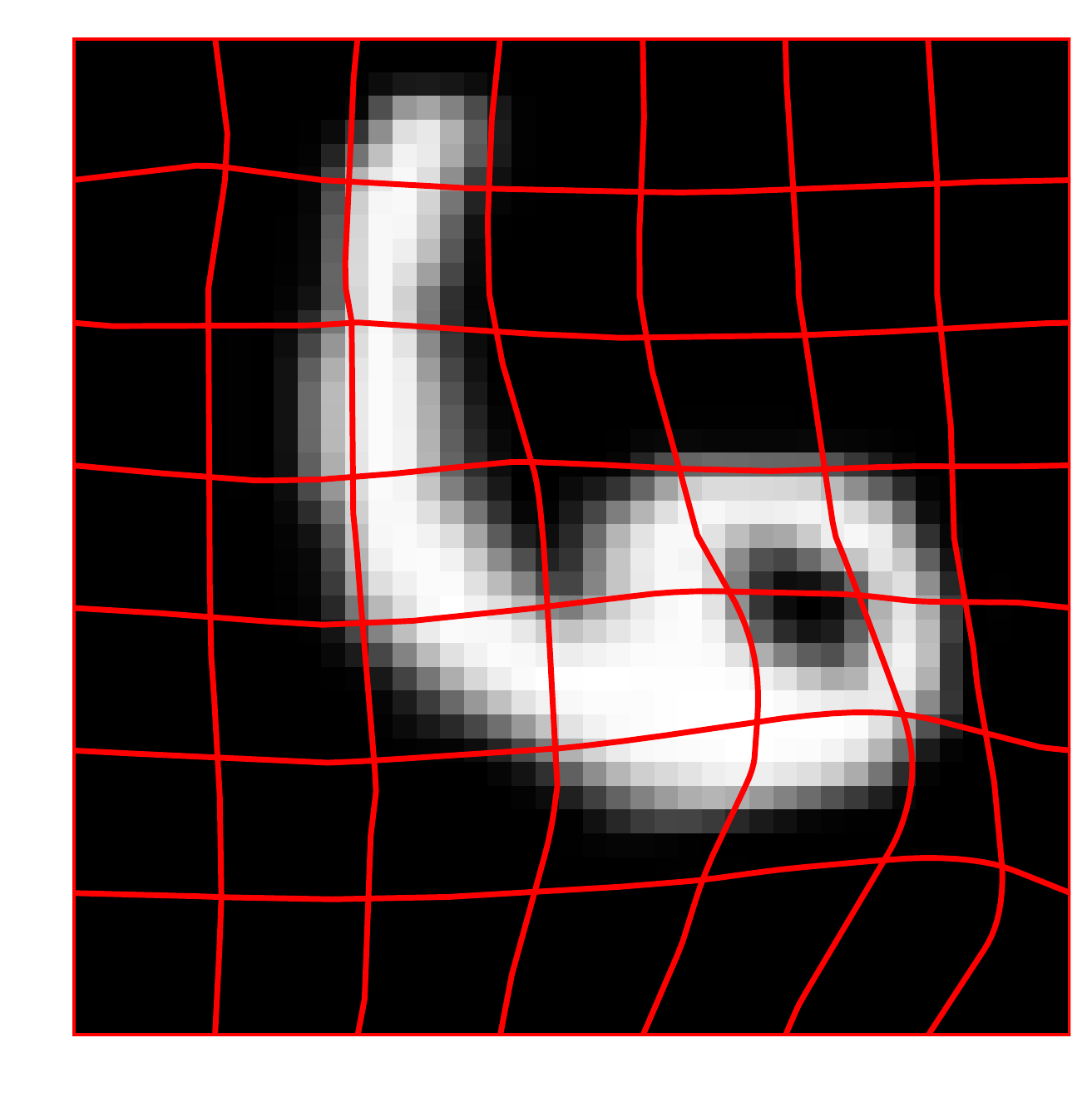}   \hspace{-3mm}
    \includegraphics[width=0.1\textwidth, clip=true, trim=4mm 4mm 4mm 4mm]{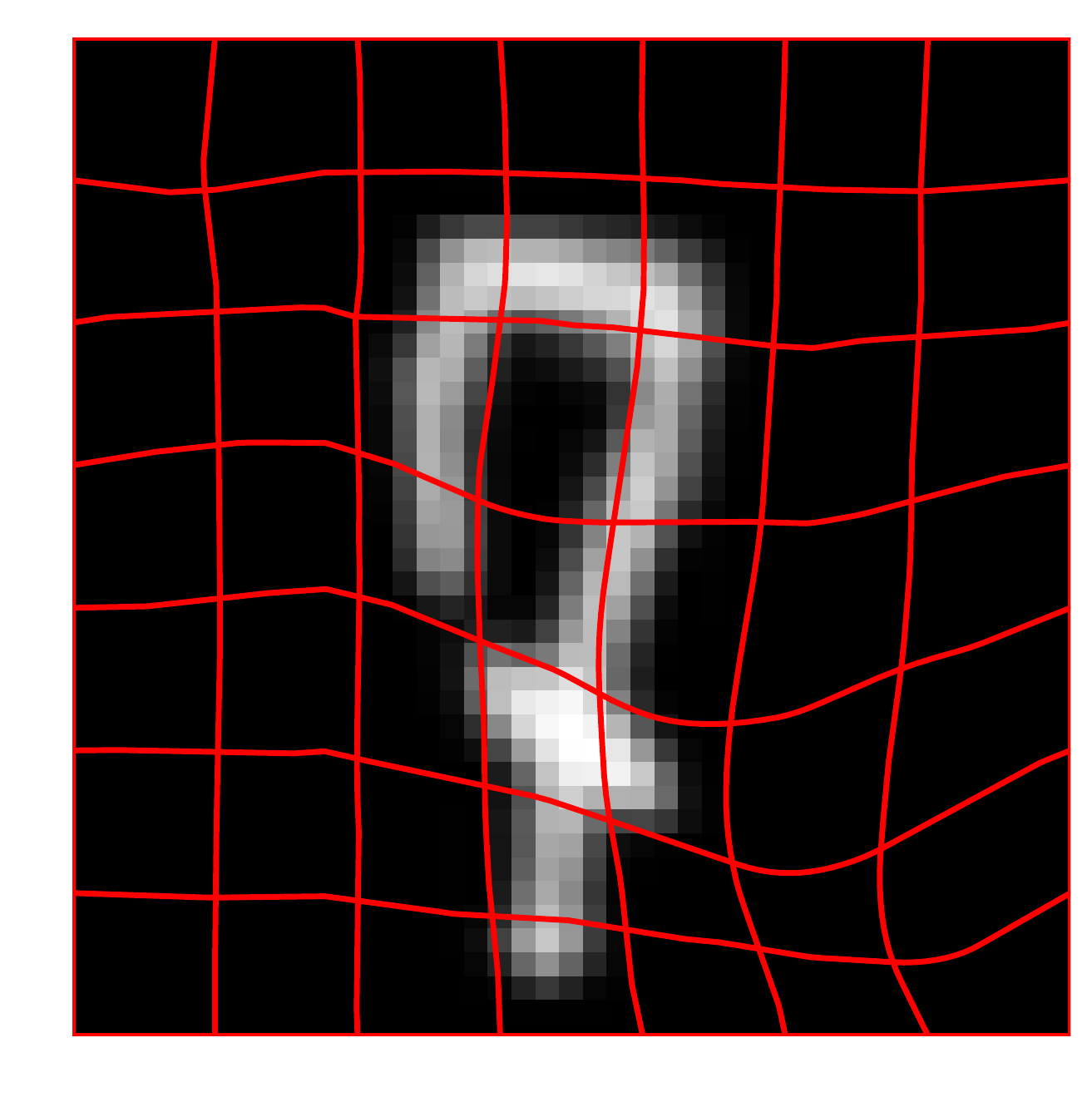} \hspace{-3mm}
    \includegraphics[width=0.1\textwidth, clip=true, trim=4mm 4mm 4mm 4mm]{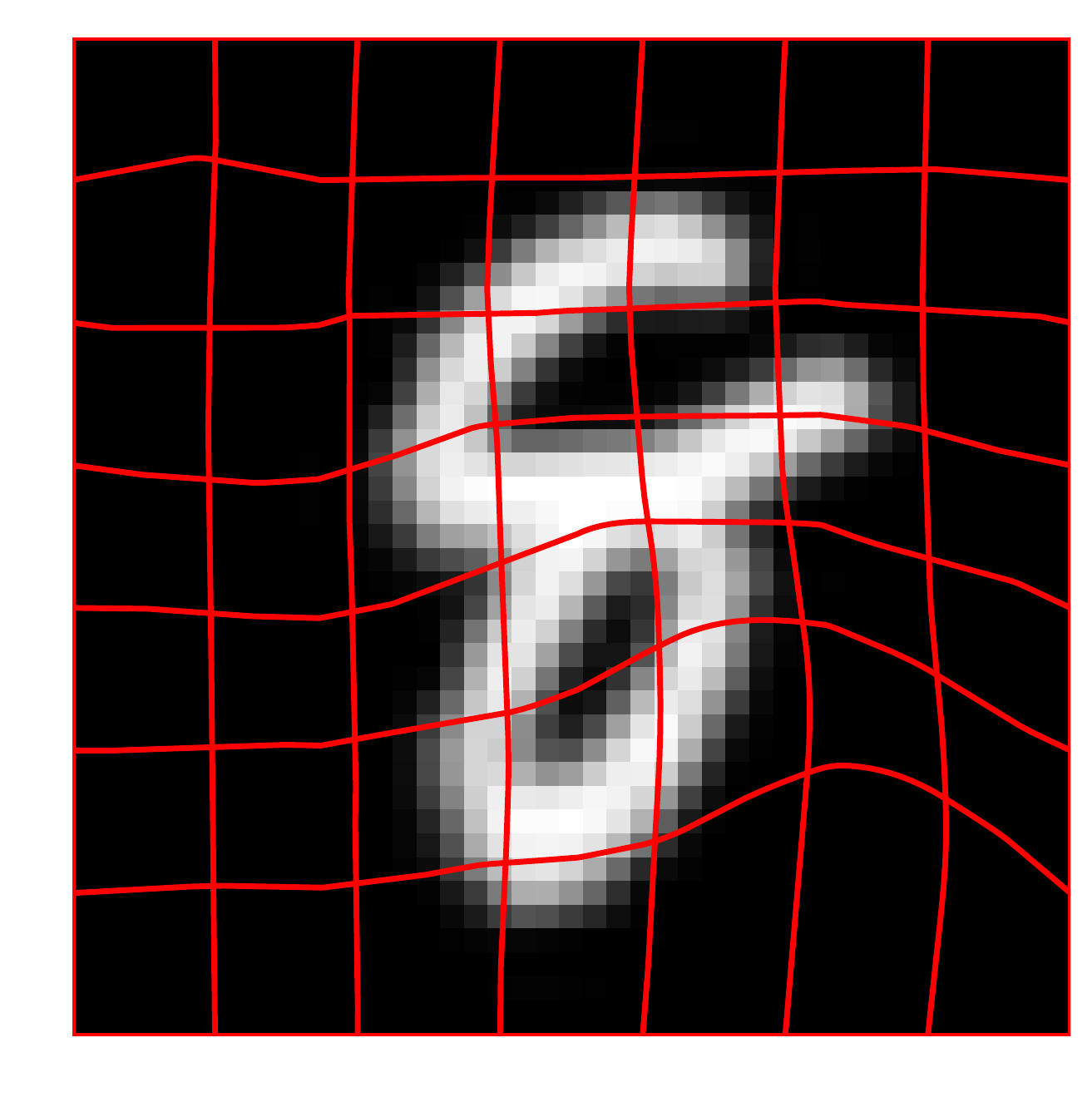} \hspace{-3mm}
    \includegraphics[width=0.1\textwidth, clip=true, trim=4mm 4mm 4mm 4mm]{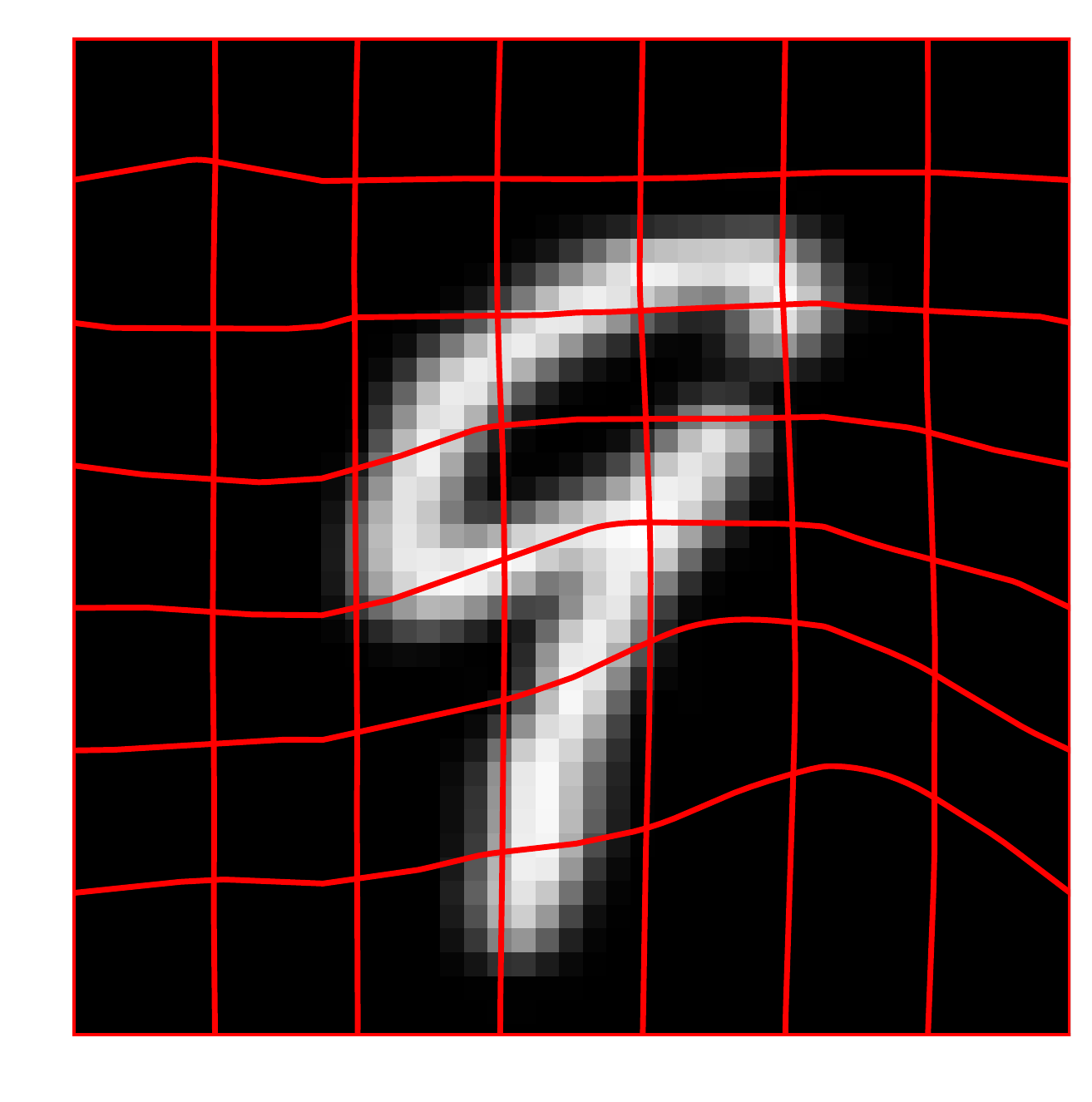}
    }
    \resizebox{0.9\textwidth}{!}{%
    \includegraphics[width=0.1\textwidth, clip=true, trim=4mm 4mm 4mm 4mm]{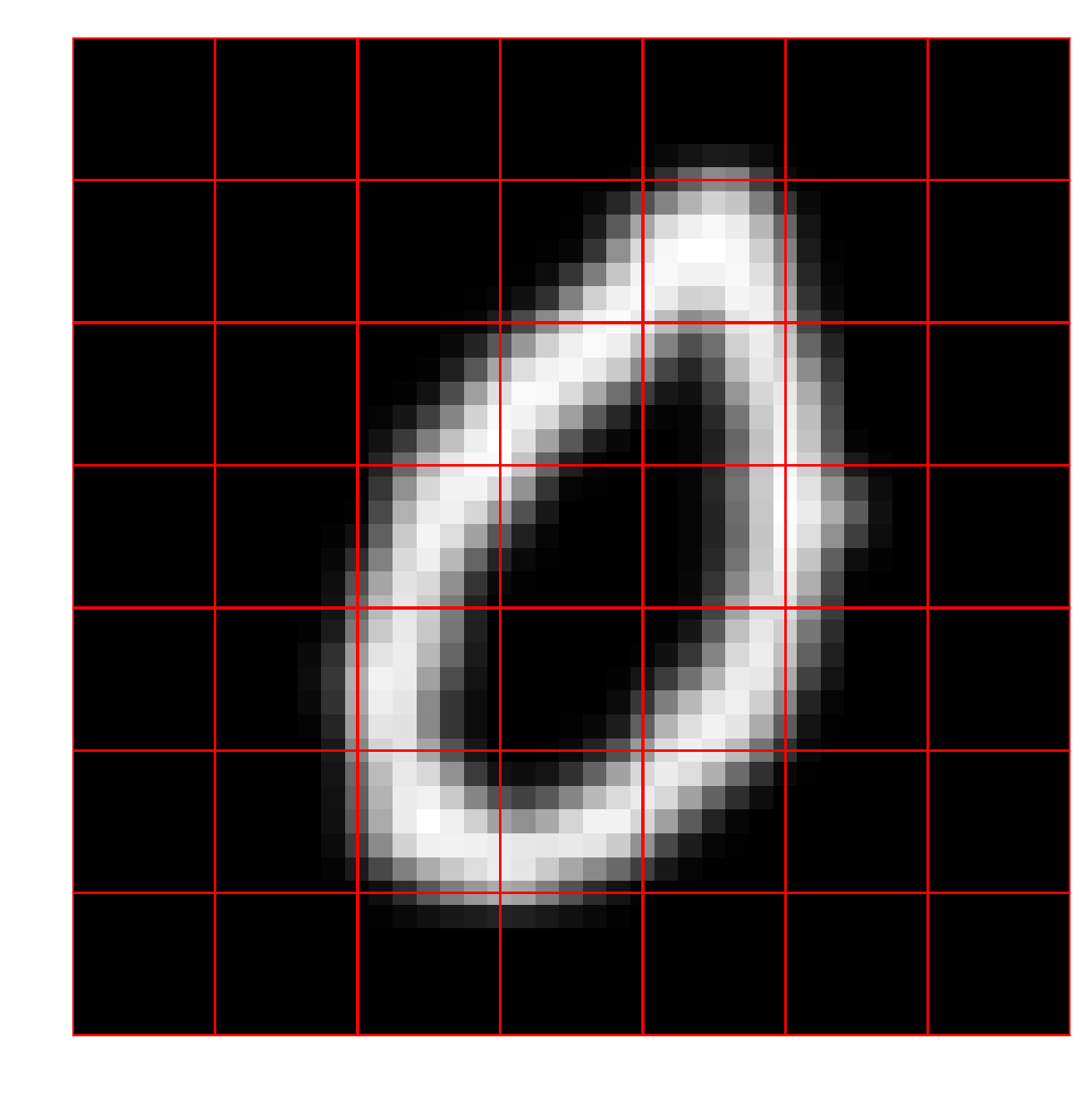}  \hspace{-3mm}
    \includegraphics[width=0.1\textwidth, clip=true, trim=4mm 4mm 4mm 4mm]{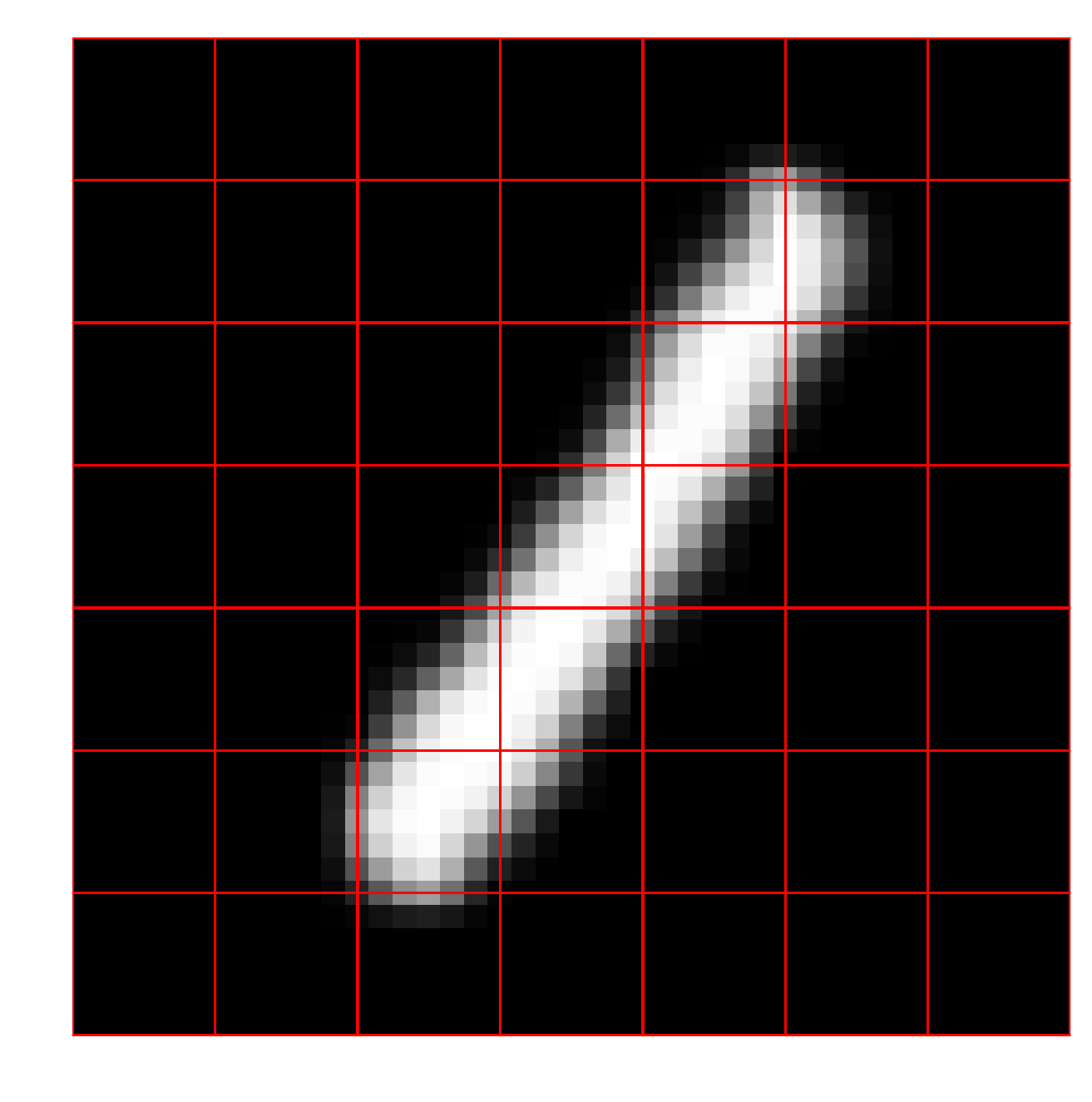}   \hspace{-3mm}
    \includegraphics[width=0.1\textwidth, clip=true, trim=4mm 4mm 4mm 4mm]{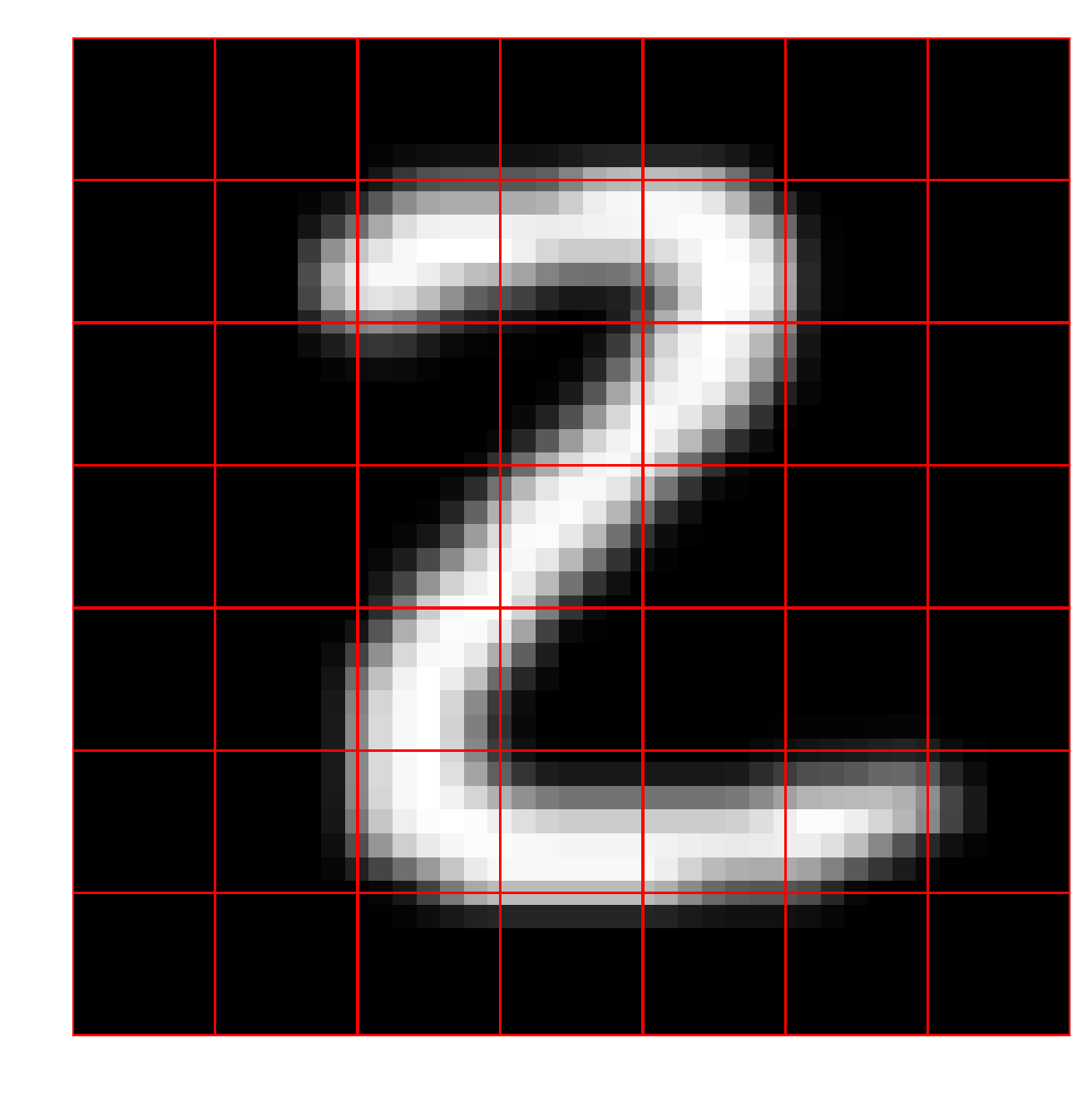}   \hspace{-3mm}
    \includegraphics[width=0.1\textwidth, clip=true, trim=4mm 4mm 4mm 4mm]{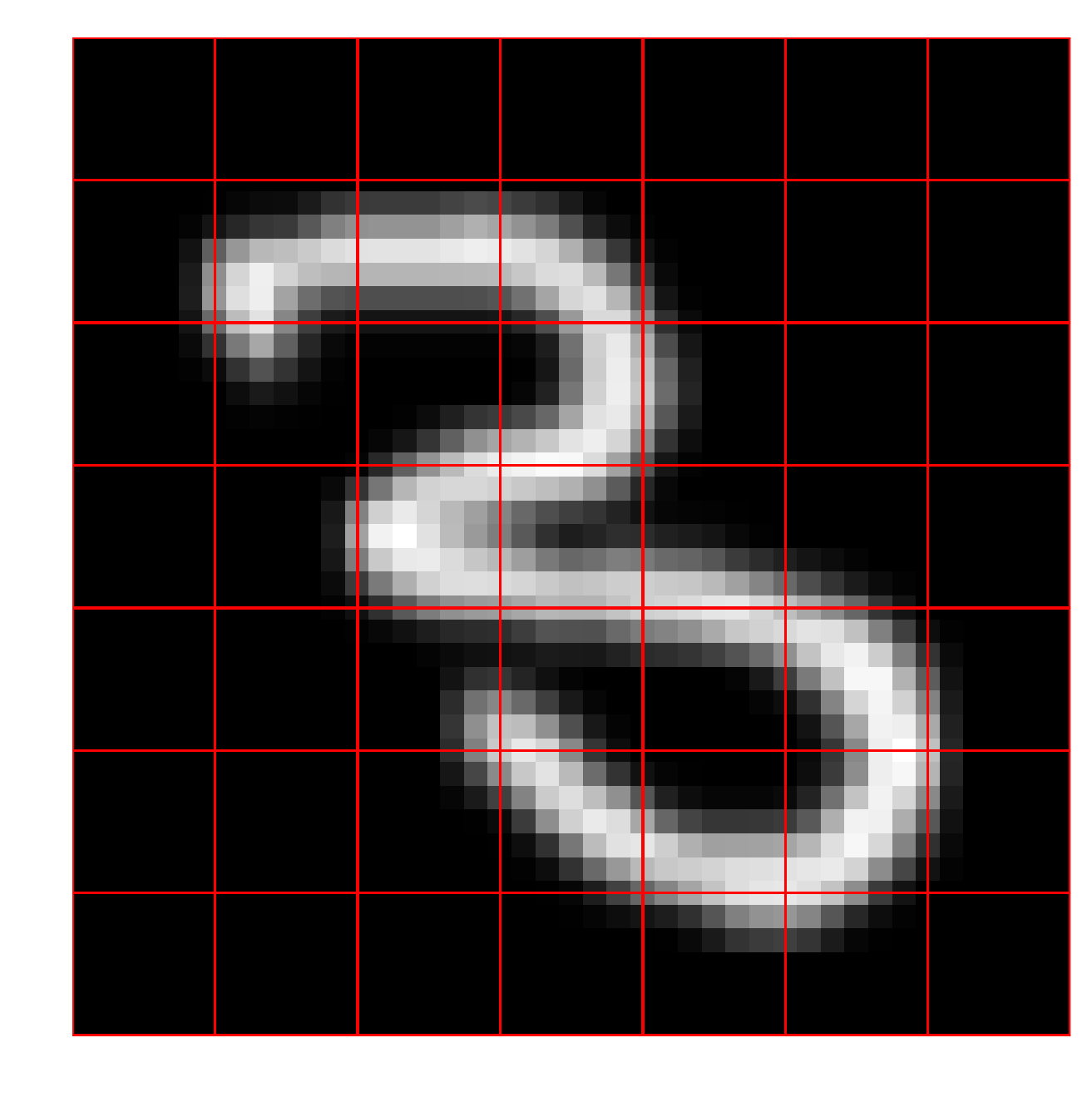} \hspace{-3mm}
    \includegraphics[width=0.1\textwidth, clip=true, trim=4mm 4mm 4mm 4mm]{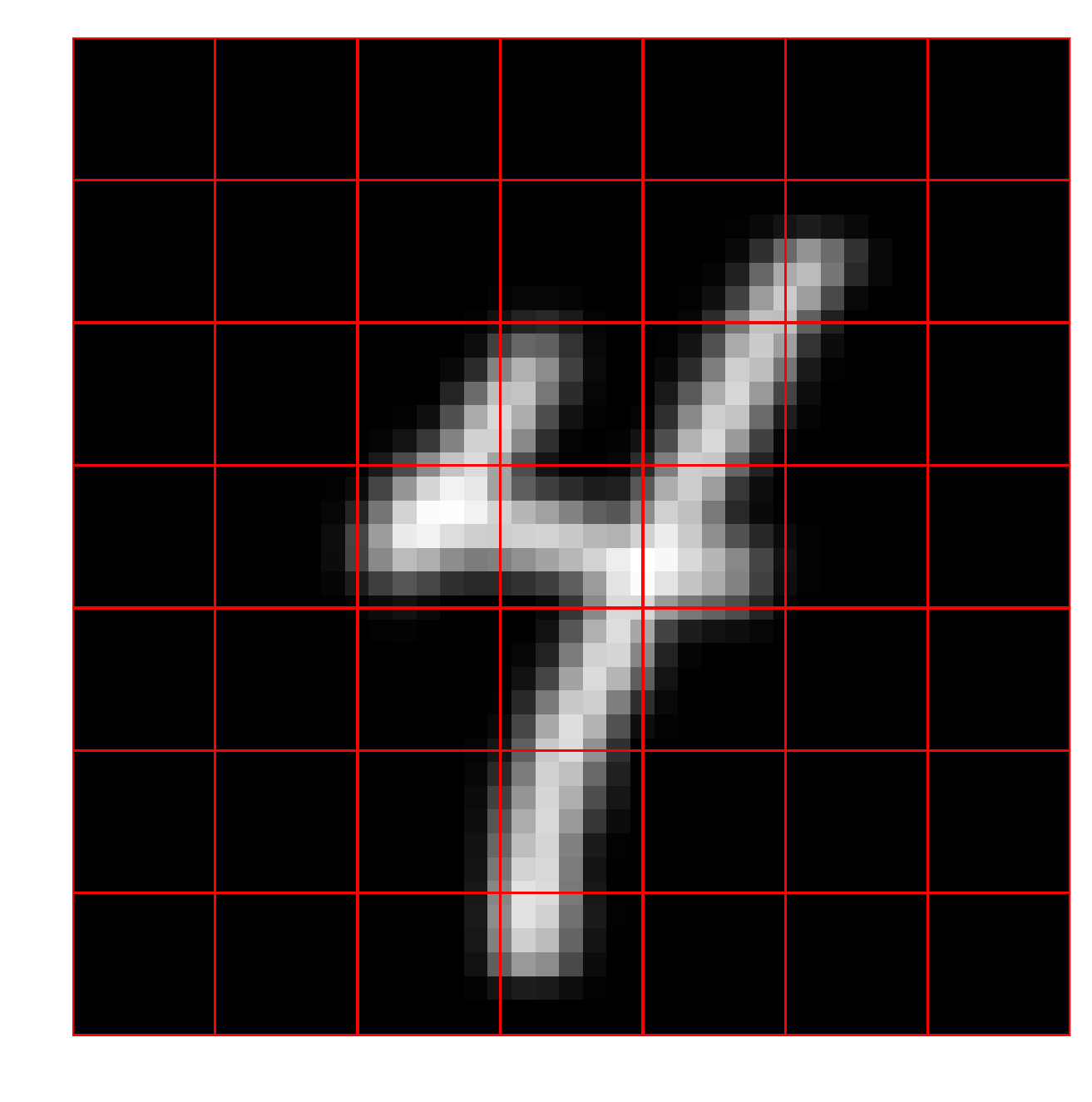}  \hspace{-3mm}
    \includegraphics[width=0.1\textwidth, clip=true, trim=4mm 4mm 4mm 4mm]{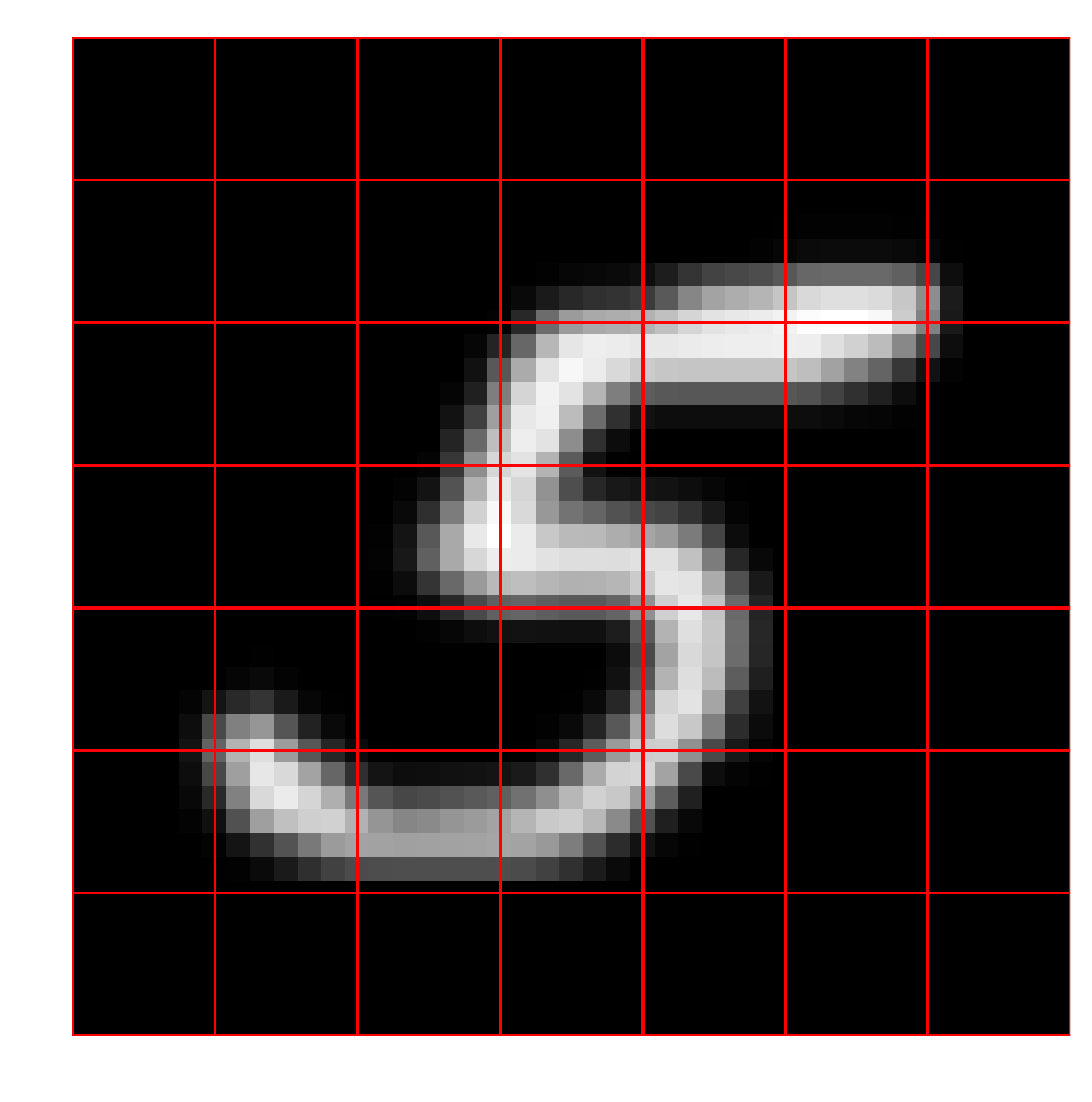}  \hspace{-3mm}
    \includegraphics[width=0.1\textwidth, clip=true, trim=4mm 4mm 4mm 4mm]{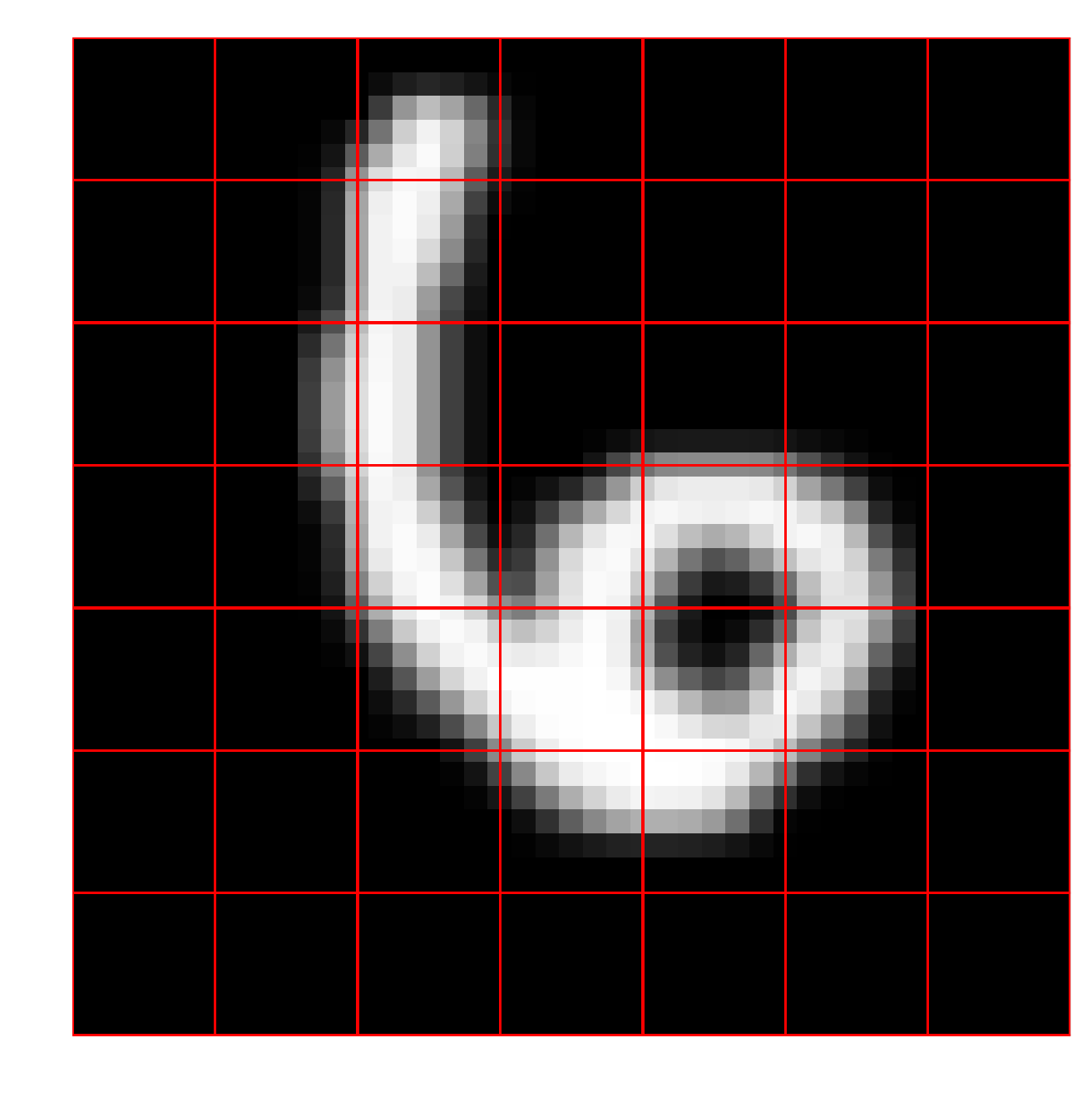}   \hspace{-3mm}
    \includegraphics[width=0.1\textwidth, clip=true, trim=4mm 4mm 4mm 4mm]{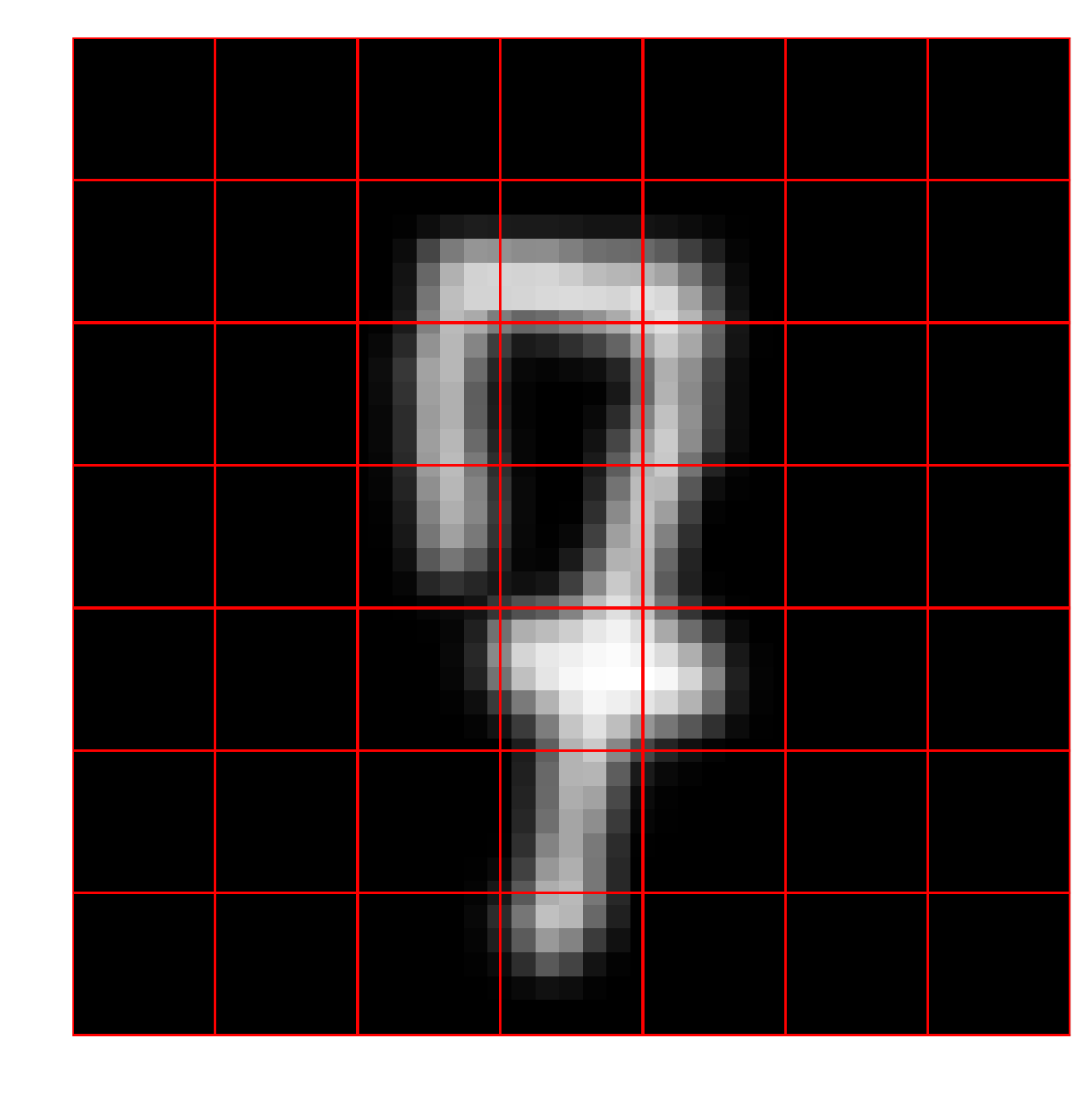} \hspace{-3mm}
    \includegraphics[width=0.1\textwidth, clip=true, trim=4mm 4mm 4mm 4mm]{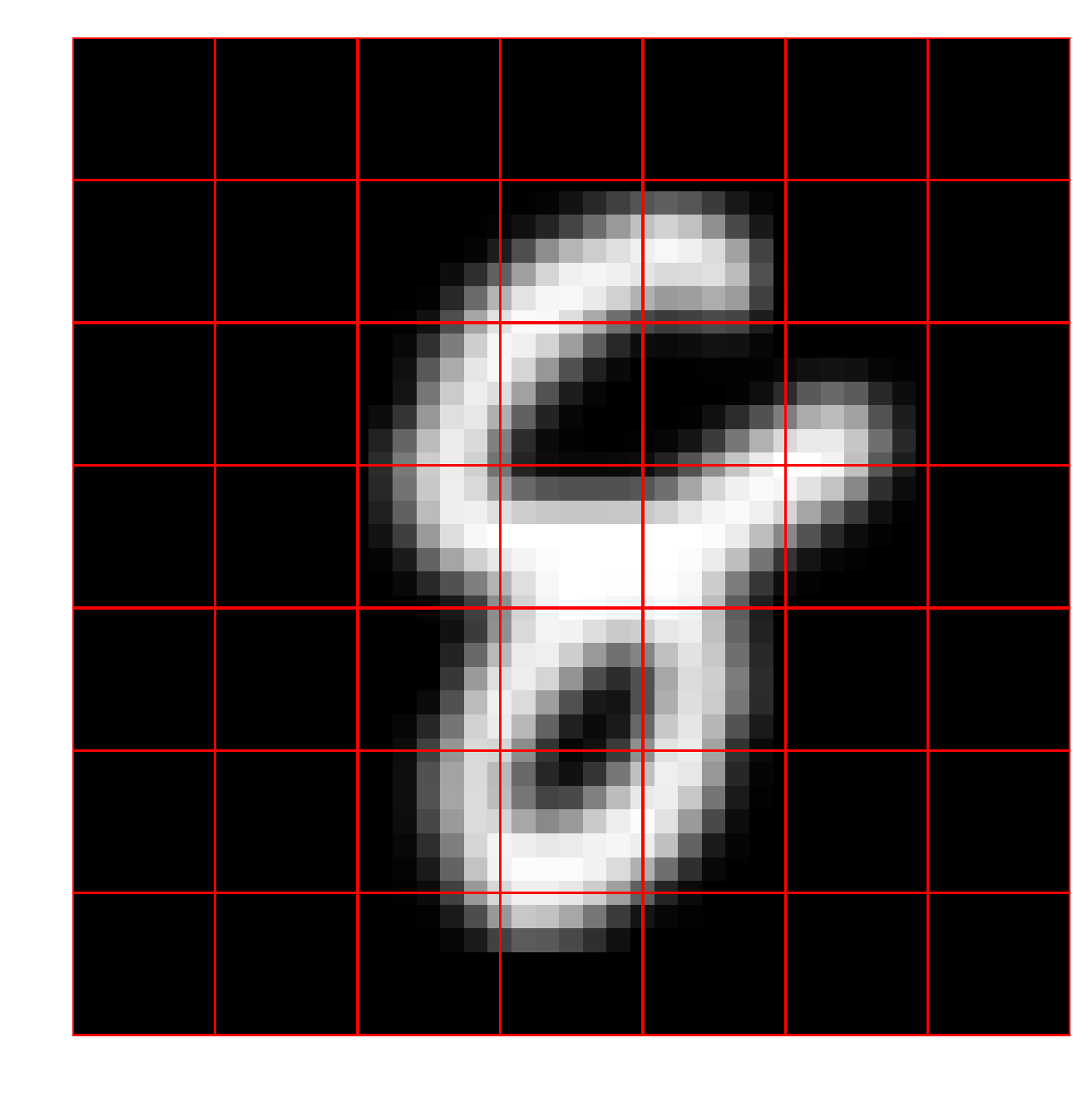} \hspace{-3mm}
    \includegraphics[width=0.1\textwidth, clip=true, trim=4mm 4mm 4mm 4mm]{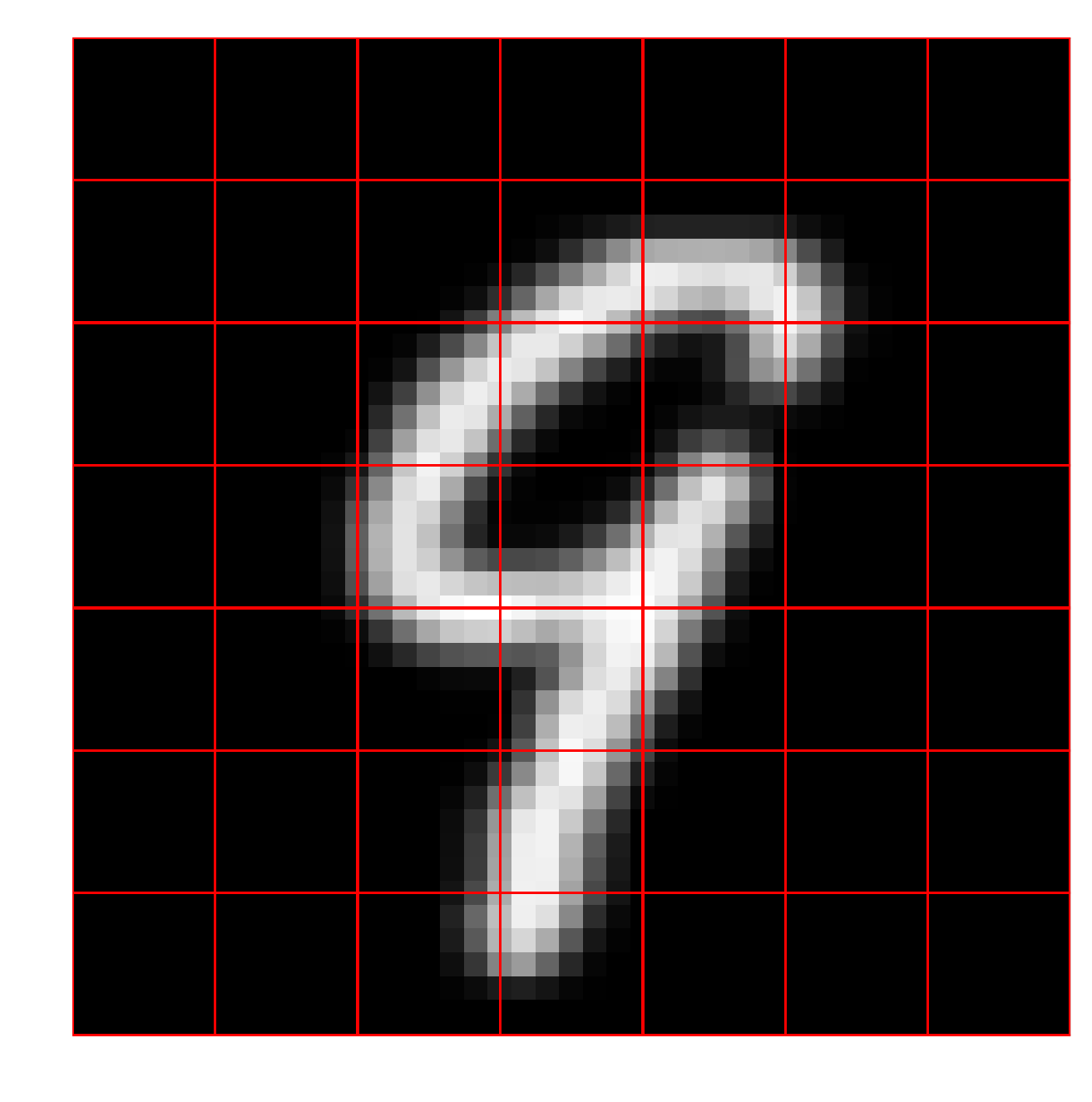}
    }
    \resizebox{0.9\textwidth}{!}{%
    \includegraphics[width=0.1\textwidth, clip=true, trim=4mm 4mm 4mm 4mm]{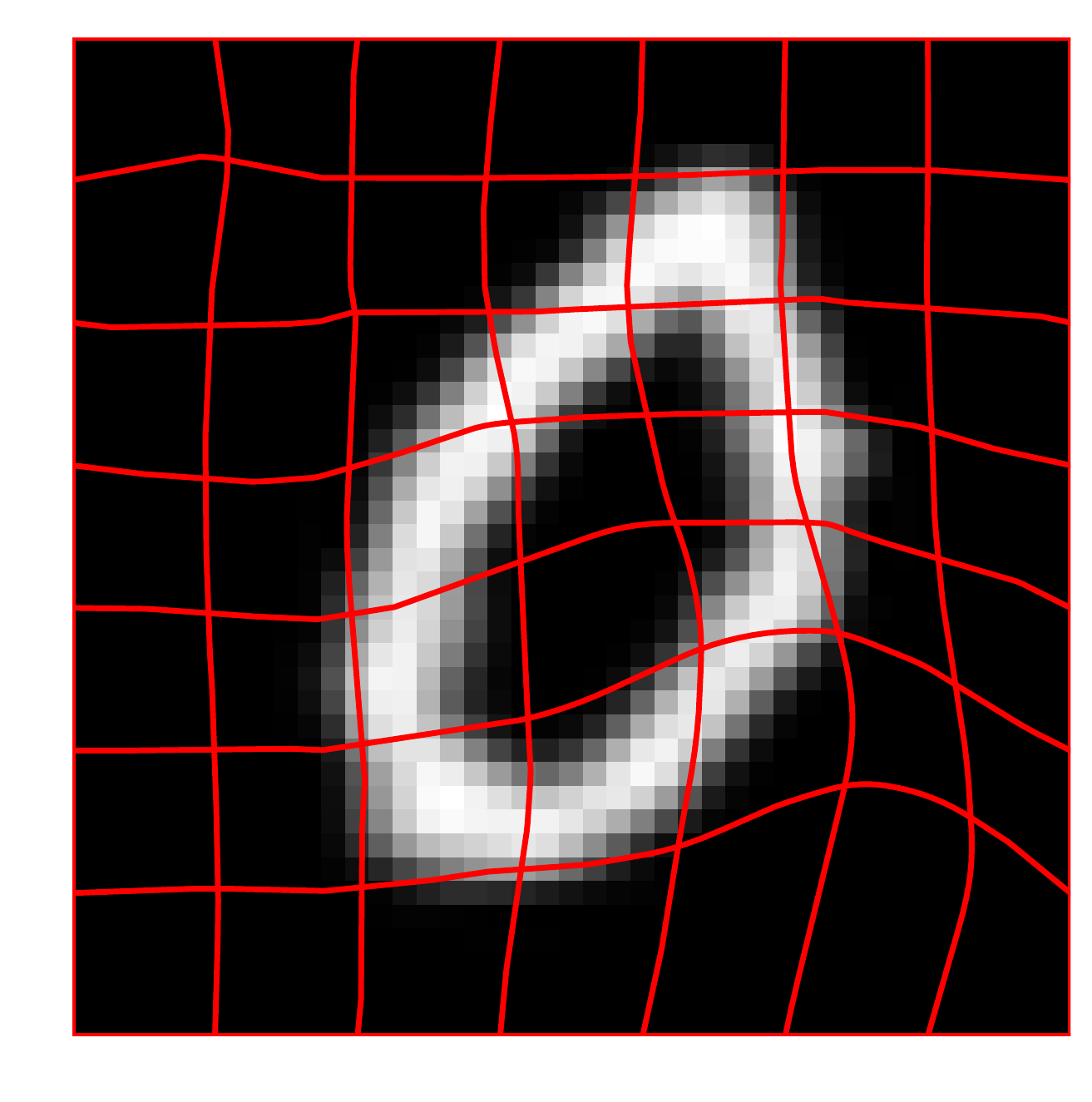}  \hspace{-3mm}
    \includegraphics[width=0.1\textwidth, clip=true, trim=4mm 4mm 4mm 4mm]{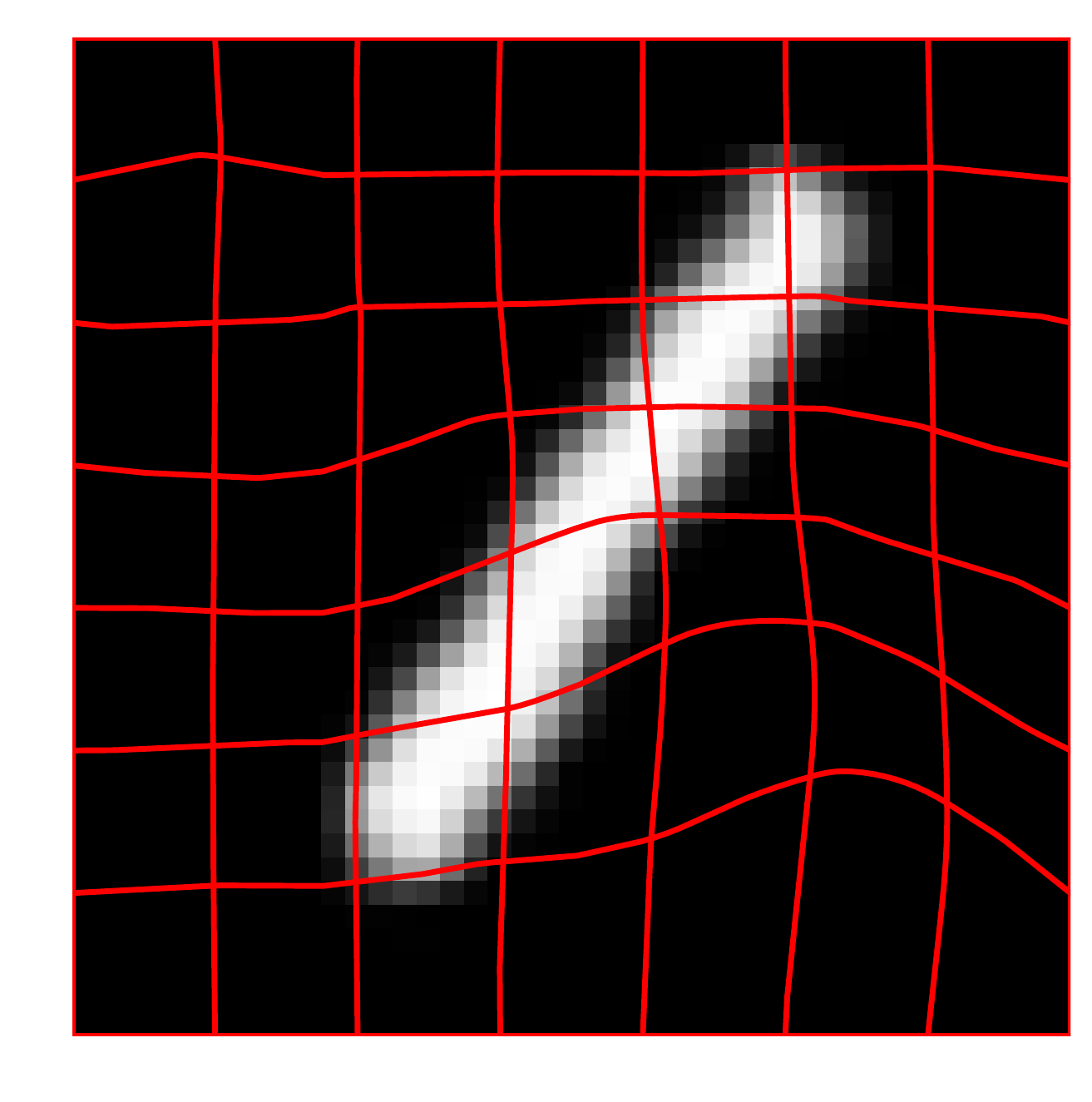}   \hspace{-3mm}
    \includegraphics[width=0.1\textwidth, clip=true, trim=4mm 4mm 4mm 4mm]{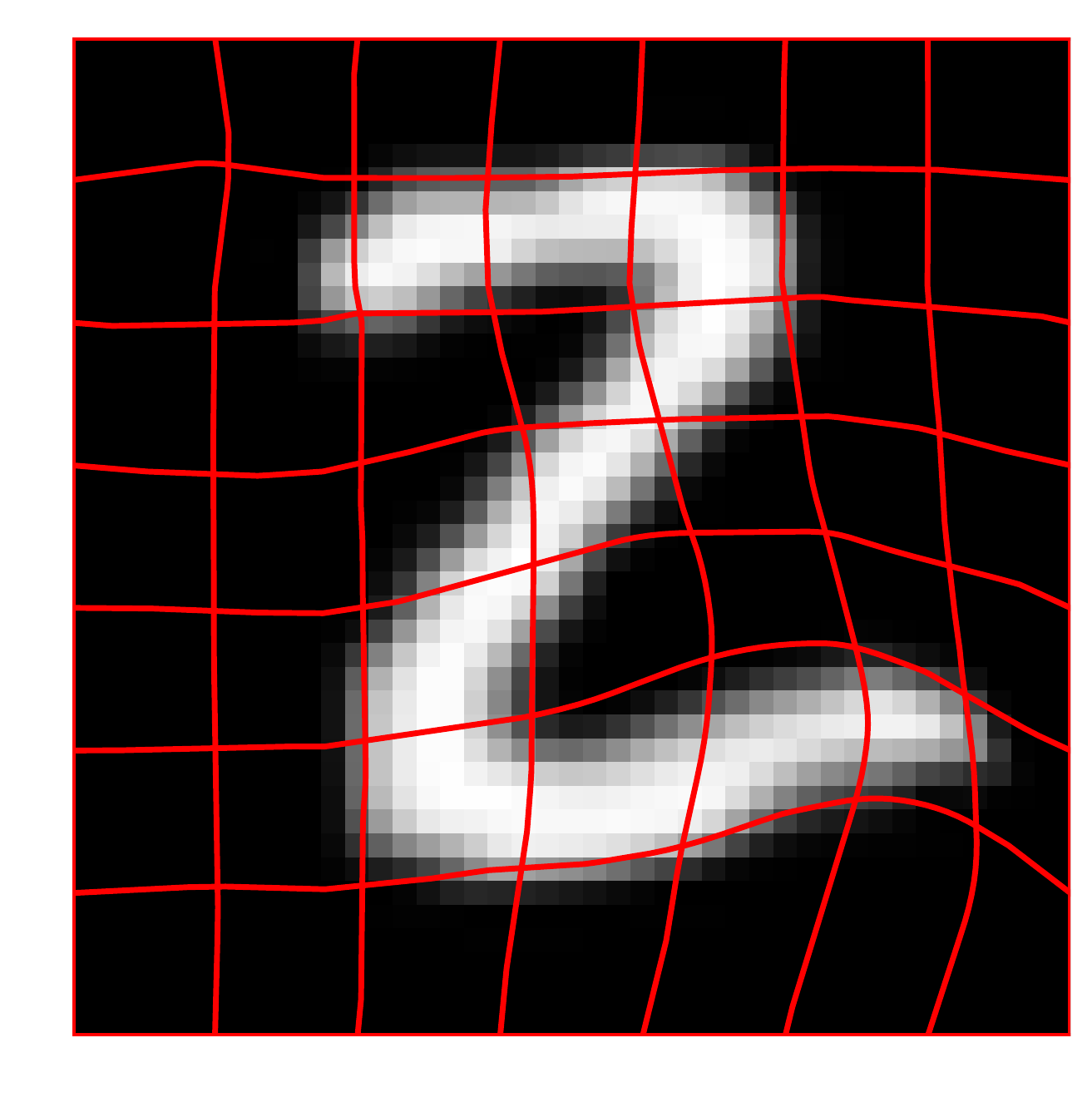}   \hspace{-3mm}
    \includegraphics[width=0.1\textwidth, clip=true, trim=4mm 4mm 4mm 4mm]{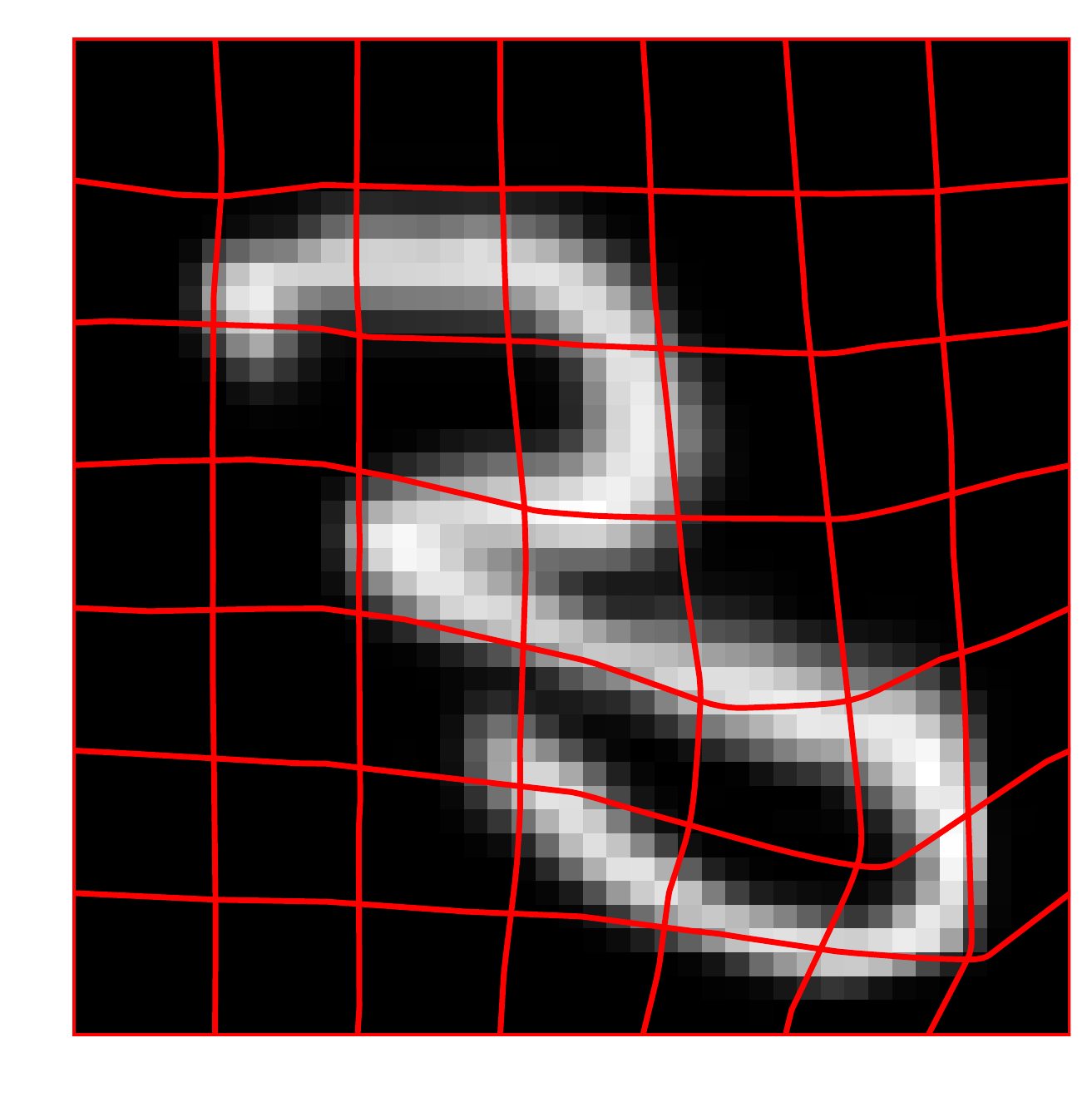} \hspace{-3mm}
    \includegraphics[width=0.1\textwidth, clip=true, trim=4mm 4mm 4mm 4mm]{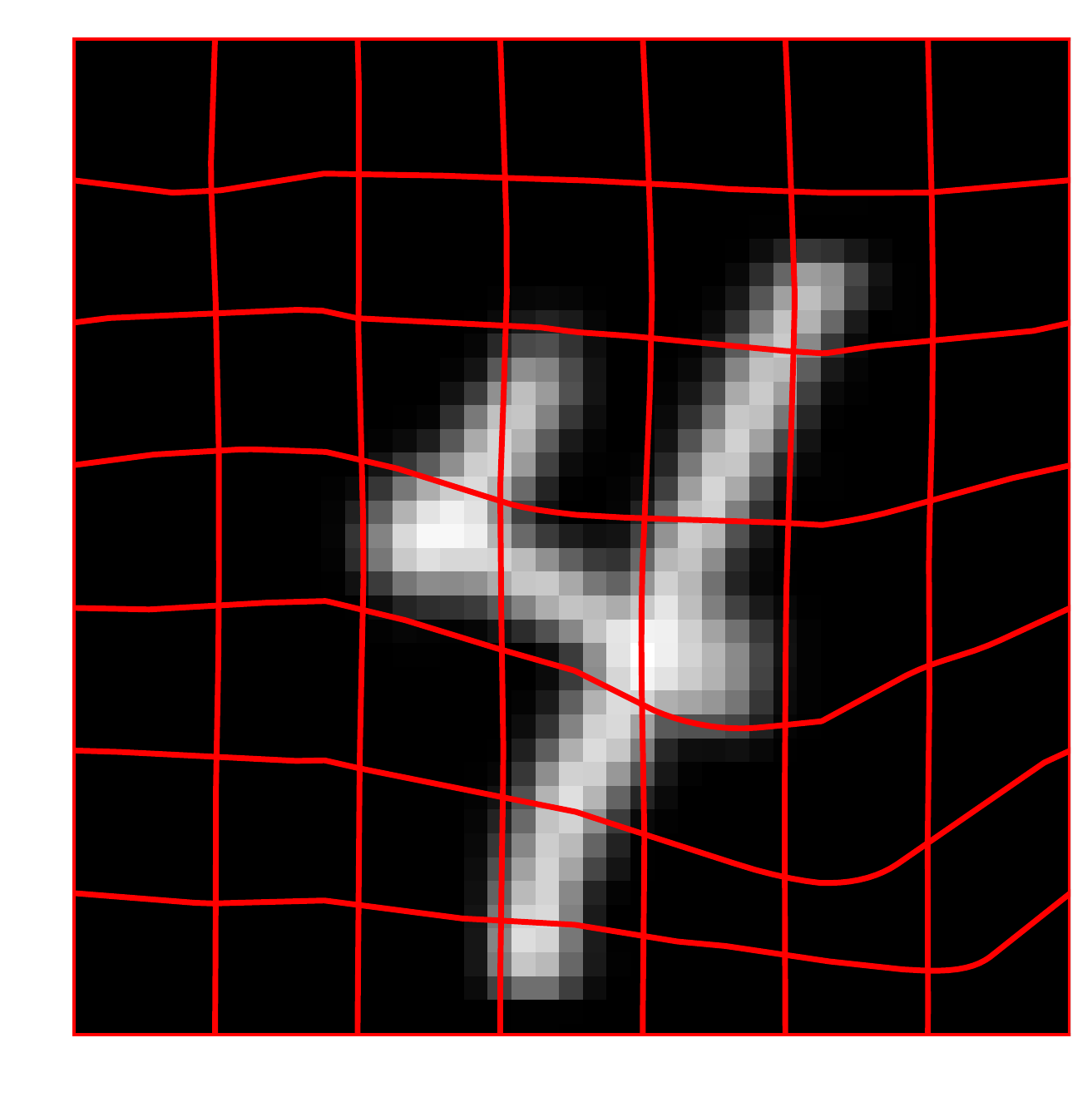}  \hspace{-3mm}
    \includegraphics[width=0.1\textwidth, clip=true, trim=4mm 4mm 4mm 4mm]{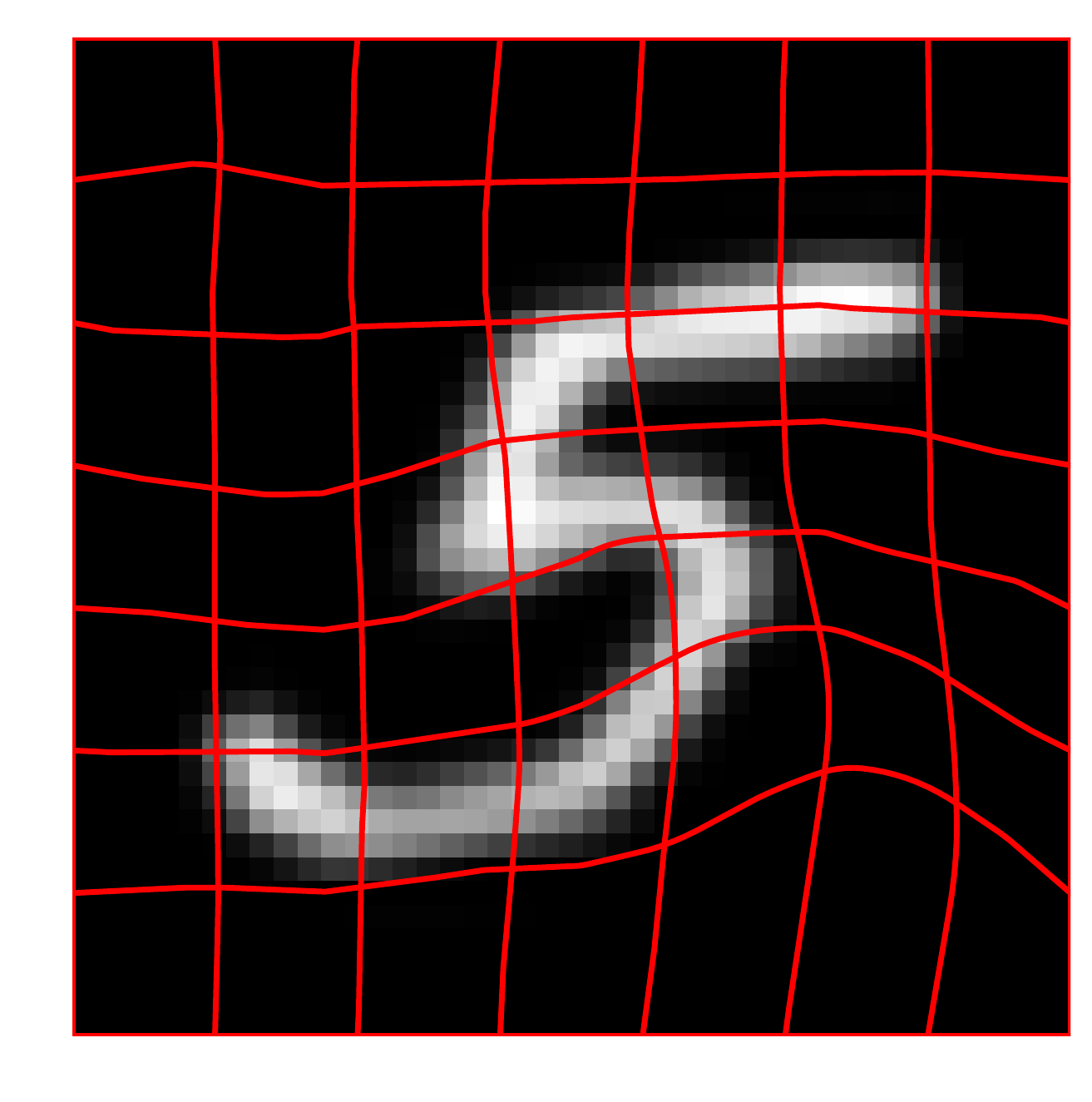}  \hspace{-3mm}
    \includegraphics[width=0.1\textwidth, clip=true, trim=4mm 4mm 4mm 4mm]{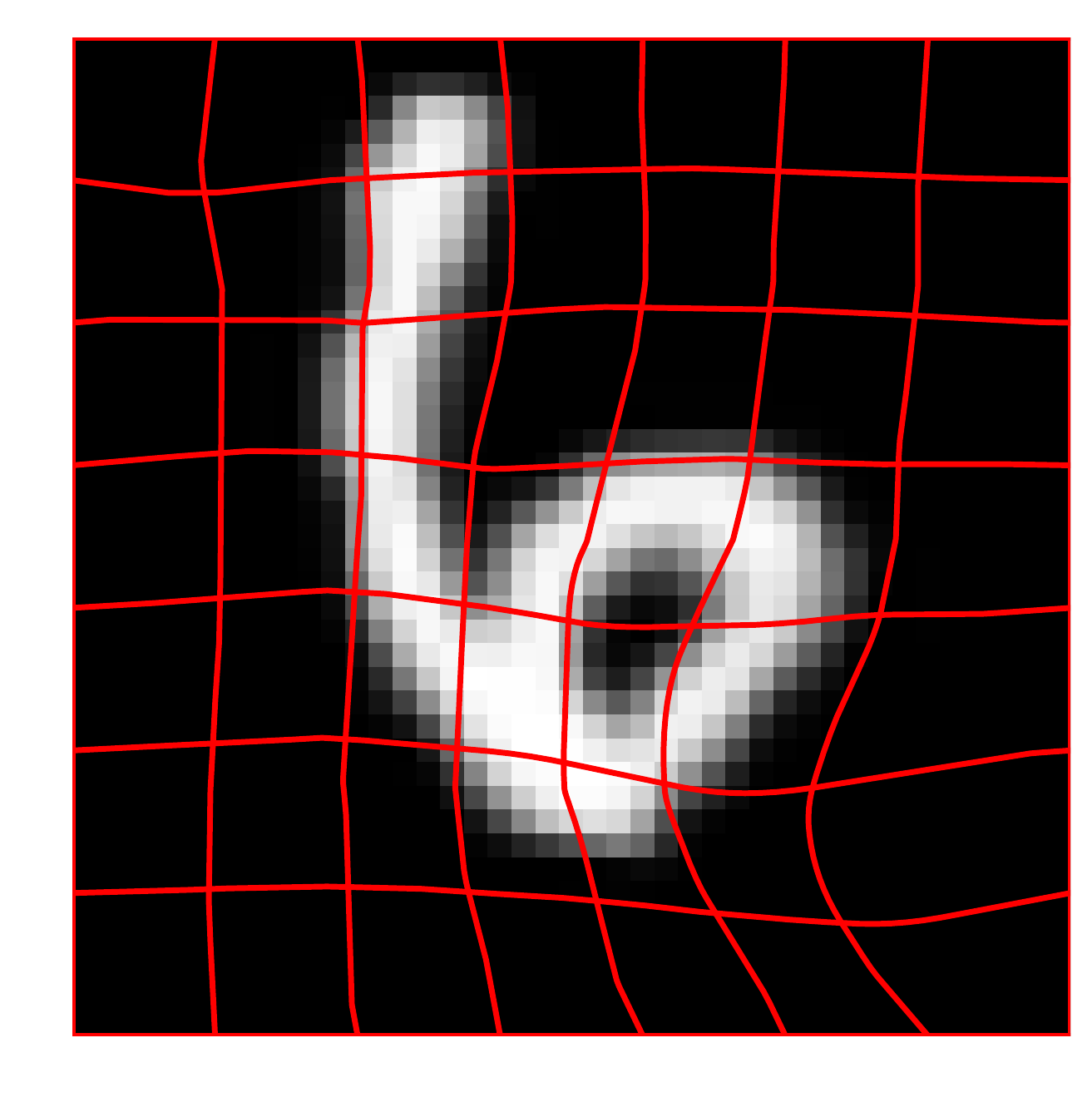}   \hspace{-3mm}
    \includegraphics[width=0.1\textwidth, clip=true, trim=4mm 4mm 4mm 4mm]{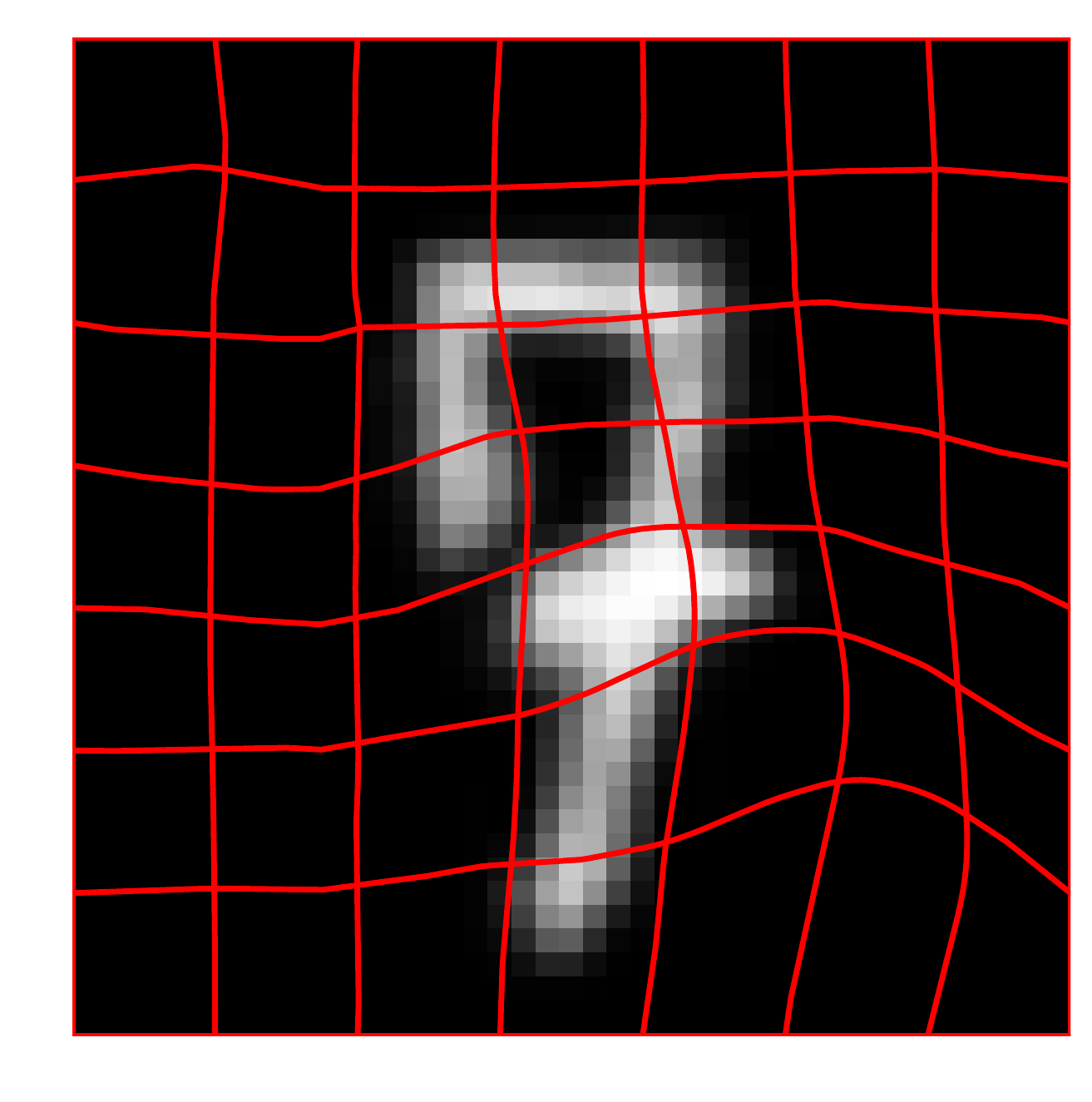} \hspace{-3mm}
    \includegraphics[width=0.1\textwidth, clip=true, trim=4mm 4mm 4mm 4mm]{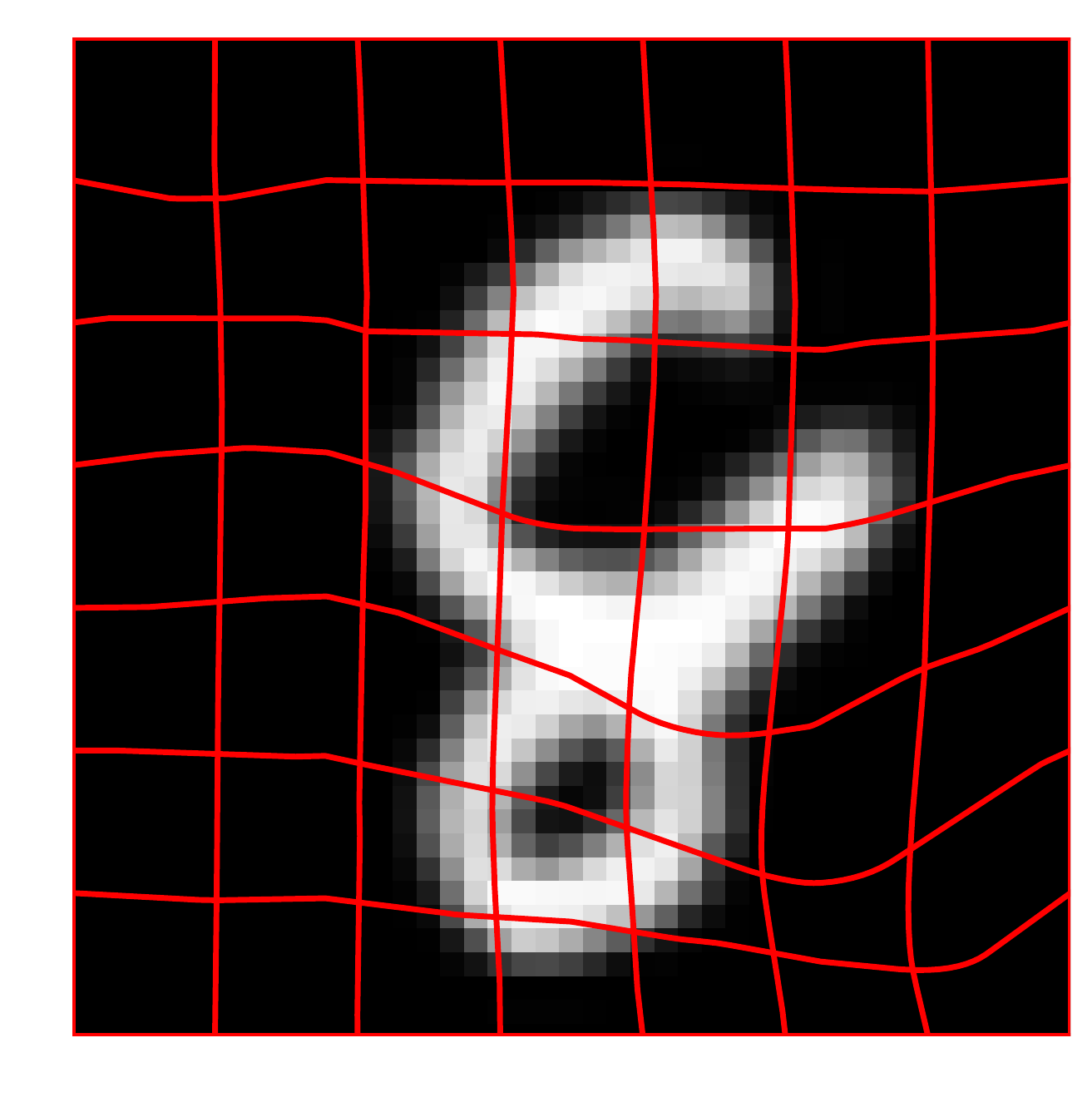} \hspace{-3mm}
    \includegraphics[width=0.1\textwidth, clip=true, trim=4mm 4mm 4mm 4mm]{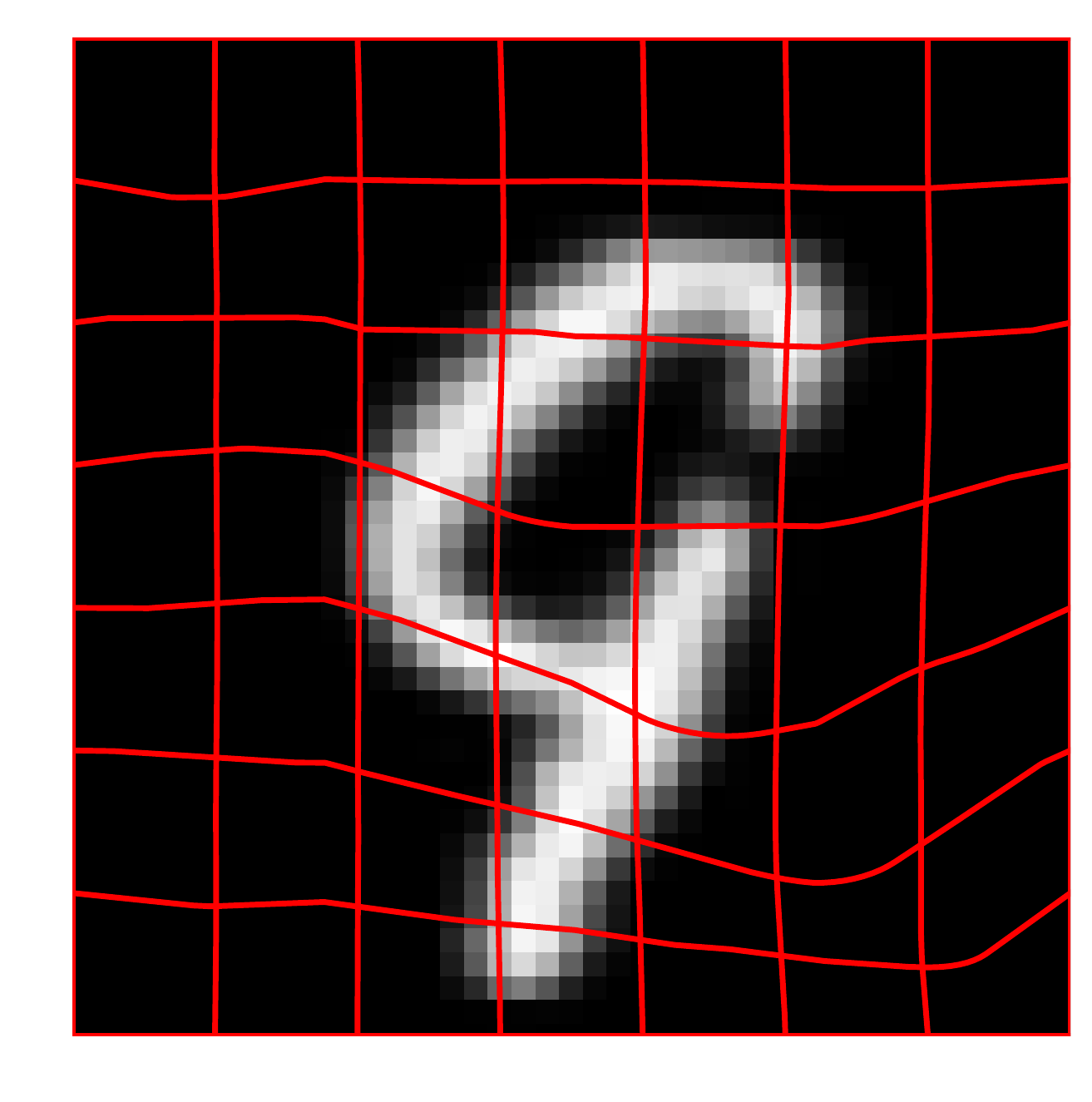}
    }
    \caption{The first principal component of the transformations of each class.
      \emph{Center row:} the mean (identity) transform.
      \emph{Top row:} the mean plus 3 standard devidations.
      \emph{Bottom row:} the mean minus 3 standard deviations.
      \textbf{See also supplementary animation and plots of further components.}}
    \label{fig:PCs}
  \end{figure*}

  \section{Experiments}
  We evaluate the learned augmentation scheme on the \emph{MNIST} data set. This allows
  us to compare with \emph{Infinite MNIST (InfiMNIST)} \cite{loosli:lskm:2007},
  which is currently the most extensive augmentation scheme available in the
  literature. The original \emph{MNIST} dataset consists of 60,000 training images and
  10,000 test images. We hold out 10,000 of the training images to form a validation
  set. The augmentations are, thus, based on only 50,000 training images.
  
  We form our augmented datasets as follows. For each of the 10 classes we sample
  500 images, and form the $K=5$ undirected nearest neighbor graph. For each edge
  we compute the transformation $T^{\btheta_{nm}}$ between the corresponding
  image pair. This gives, on average, 2940 transformations per class. We fit 
  a zero-mean multivariate Gaussian to the tangential
  representations $\bv^{\btheta_{nm}}$ of the transformations, and generate new
  data as described in Sec.~\ref{sec:augmentation}. With this approach we generate two new training
  sets: The \emph{AlignMNIST} set is generated by uniformly sampling images from
  the 50,000-element training set and applying transformations sampled from $p(T^\btheta | y)$
  to generate 1,000,000 images per class. The \emph{AlignMNIST500} set is generated
  similarly except we only sample images from the set of 500 images from which
  transformations were estimated; again we sample 1,000,000 images per class.
  The \emph{AlignMNIST500} set is, thus, generated from only 500 images per class and
  allows us to experiment with the effect of learned augmentation in small datasets.
  For comparative purposes, we also sample 1,000,000 images per class from
  \emph{InfiMNIST} based on the 50,000-element training set; we further generate
  \emph{InfiMNIST500} by sampling from \emph{InfiMNIST} using only the same 500 images
  per class as for \emph{AlignMNIST500}. 
  
  As an initial experiment we consider a simple nearest-neighbor classifier on
  \emph{MNIST} (test error: $3.1\%$), \emph{InfiMNIST} (test error: $2.6\%$),
  and \emph{AlignMNIST} (test error: $1.4\%$). This gives a hint that the learned
  augmentation scheme captures invariances 
  that were missed
  in the laborious
  manual specification behind \emph{InfiMNIST}.
  
  Next, for each dataset we train a \emph{multilayer perceptron (MLP)} with hyperparameters
  estimated with cross-validation on 
  the held-out validation set.
  The best set of hyperparameters is then used to train a network on the entire training set.
  For both \emph{InfiMNIST} and \emph{AlignMNIST}, we experience no overfitting problems due
  to the variety of the input samples. In fact, for both augmentation schemes,
  the network training converges even before the network has seen the entire
  dataset. Therefore, we train the final model without a scheme for early
  stopping by simply doubling the amount of weight updates used for validation.
  For both augmentation schemes, we achieve best performance using networks
  consisting of 3 hidden layers with 2048 units each.
  To speed up the training, we optimize using a stochastic gradient descent with
  momentum and employ rectified linear units as activation functions.

  \begin{figure}
    \centering
    \includegraphics[width=\columnwidth]{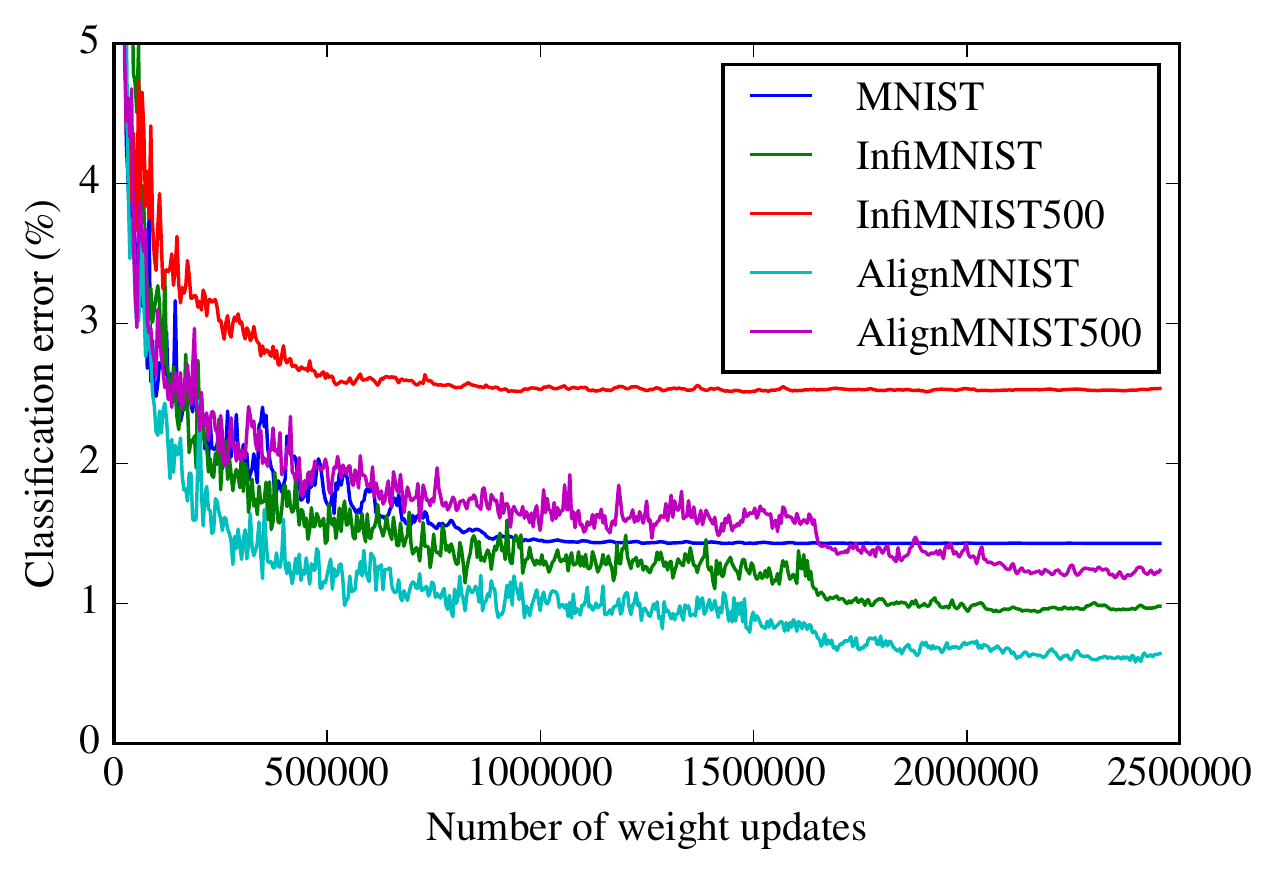}
    \caption{Learning curves for different training sets. The 
      classification error is calculated on the MNIST test set.}
    \label{fig:learn_curve}
  \end{figure}

In addition to the MLP we train a convolutional network (ConvNet) in a similar fashion on the datasets.
The ConvNet consists of 2 convolutional + pooling layers followed by a fully-connected hidden layer of 512 units.
To avoid overfitting, we rely on early stopping as determined on the validation set.

  Figure~\ref{fig:learn_curve} shows the learning curves of the MLP on the different
  datasets, while Fig.~\ref{fig:res} summarizes the results.
  It is evident that \emph{AlignMNIST} gives rise to the best predictive model,
  with a significant improvement over \emph{InfiMNIST}. This is clearer
  when we consider datasets generated from only 500 data points per class.
  The performance of the MLP trained on \emph{AlignMNIST500} is less than $0.2$
  percent-points worse than the MLP trained on the entire \emph{InfiMNIST},
  and $1.6$ percent-points better than the model trained on \emph{InfiMNIST500}.
  The additional performance gains attained by the ConvNets confirm that the
  benefits of our augmentation scheme generalize to networks with spatial structure.

  Note that we are unable to reproduce the ConvNet results on \emph{InfiMNIST}
  previously reported by \cite{simard2003best}.
  Nonetheless, we believe that the ConvNet numbers are representative as the same
  implementation and evaluation scheme are used for all datasets.
  
  The reported results for \emph{InfiMNIST500} and \emph{AlignMNIST500} should be
  compared with the current state-of-the-art for MNIST classification based on
  small training sets. Amit et al. \cite{Amit:ijcv:2007} report an error rate of $1.5\%$
  when trained on 500 images per class. Both our MLP and our ConvNet significantly
  improve upon this, which demonstrate that a good augmentation scheme can make
  deep learning a viable option even for fairly small datasets.
  
  \begin{figure}
    \begin{tabular}{lcc}
                             & \textbf{MLP}           & \textbf{ConvNet} \\
      \textbf{Dataset}       & \textbf{Test error}    & \textbf{Test error} \\
      \hline
      \emph{MNIST}           & $1.42\% \pm 0.055$     & $0.65\% \pm 0.08\%$ \\ 
      \emph{InfiMNIST}       & $0.89\% \pm 0.079$     & $0.49\% \pm 0.04\%$ \\
      \emph{AlignMNIST}      & $0.58\% \pm 0.022$     & $0.44\% \pm 0.02\%$ \\ \hline
      \emph{InfiMNIST500}    & $2.59\% \pm 0.050$     & $1.04\% \pm 0.07\%$ \\
      \emph{AlignMNIST500}   & $1.06\% \pm 0.047$     & $0.84\% \pm 0.05\%$	
    \end{tabular}
    \caption{Final test error on the different datasets.}
    \label{fig:res}
  \end{figure}

%

%

  \section{Discussion}\label{sec:discussion}
  Our basic idea is simple: \emph{instead of manually specifying data augmentation
  schemes, we build a statistical model of the transformations found within a
  given class, and use this to augment the dataset}. The practical implementation
  of this idea, however, requires some care. A naive approach would build a statistical model 
  over dense displacement fields. This would, however,
  be a very high-dimensional model that would allow for too much flexibility.
  Such an approach would also require aligning substantially more image pairs,
  which would increase computational demands.
  Constraining the transformations to lie in a finite-dimensional space of diffeomorphisms 
  not only drastically lowers dimensionality but also ensures that samples from the probabilistic
  model are well-behaved transformations. 
  As the first step of the proposed method involves solving multiple problems of inference over latent diffeomorphisms,
  we crucially depend on the availability of an efficient-yet-highly-expressive representation of diffeomorphisms;
  fortunately, this is an active area of research with several recent successes
  \cite{freifeld2015transform, zhang:ipmi:2015, Arsigny:BIR:2006}.
  
  We find that a learned augmentation scheme allows for significantly smaller
  training sets. In particular, we report state-of-the-art results on small
  subsamples of MNIST. This observation potentially allows very large models to be
  trained on fairly small datasets; we observe that the learned augmentation
  scheme using a fraction of the training data not only outperforms a classifier
  trained on the entire dataset but also does significantly better than manually-specified augmentations.
  As there are $\mathcal{O}(N_y^2)$ potential pairs to be aligned in a class of $N_y$
  observations, it is generally possible to get access to enough transformations
  to learn a good augmentation scheme.
  A further benefit of learned augmentation schemes is that different schemes
  may be applied within different classes --- this is generally impractical
  to do in manually-designed schemes. 
  
  A limitation of our approach is that we must be able to align observations
  in order to build statistical models of the deformations found within the dataset.
  We exemplify our idea on the MNIST data where observations often have well-defined
  alignments. Our approach is, however, not limited to MNIST:
  \begin{itemize}
    \item Image alignment is a routine task in many medical imaging tasks, such as the analysis
      of \emph{magnetic resonance images (MRI)} \cite{darkner:pami:2013, miller2004computational, beg:ijcv:2005},
      \emph{X-ray Computed Tomography (CT)} \cite{Rueckert:2011, mattes2003pet},
      \emph{Positron Emission Tomography (PET)} \cite{studholme1997automated, studholme1995multiresolution} and
      \emph{mammograms} \cite{hipwell2007new, nielsen:jmiv:2008}.
      Our work directly extends to these domains.
    \item We make similar observations for time-series data such as
      acoustic signals~\cite{cui:icassp:2014, hinton2012deep}. Here
      \emph{dynamic time warping (DTW)} \cite{Bellman:DTW:1959} is often used as pre-processing
      to remove differences in the temporal speed of individual signals.
      The CPAB representation readily provides a replacement for DTW and our approach
      can be used to augment the datasets. It is worth noting that DTW itself is poorly-suited
      for such a task as a generative statistical model of deformations learned from estimated DTWs will not
      yield new invertible transformations even if all estimated DTWs happen to be invertible.
      We avoid such issues by relying on the CPAB representation, which ensures
      diffeomorphic deformations.
    \item \emph{Mesh alignment} is also standard pre-processing step in the analysis
      of three-dimensional meshes \cite{bogo:cvpr:2014, hirshberg:eccv:2012, anguelov2005scape, allen2002articulated}.
      As deep models are beginning
      to appear for three-dimensional data~\cite{wu:cvpr:2015} it would be interesting
      to combine them with learned augmentation schemes.
  \end{itemize}

  In summary, the main contribution of this paper is the first
  approach for learning data augmentation schemes. Our results show that
  it is beneficial to learn statistical models of transformations over manually
  specifying them. We do not find it surprising that \emph{learning} beats
  \emph{hand-crafting}.

  \subsubsection*{Acknowledgements}
  S.H.\ is funded by the Danish Council for Independent Research, Natural Sciences.  
  O.F.\ and J.W.F.\ are partially supported by U.S.\ Office of Naval Research MURI program, award N000141110688,
  and VITALITE, which receives support from U.S. Army Research Office MURI, award W911NF-11-1-0391.
  
  \small{
    \bibliographystyle{abbrv}
    \bibliography{paper}
  }
  
\end{document}